\documentclass{article}

\usepackage[T1]{fontenc}
\usepackage[utf8]{inputenc}
\usepackage{mathtools}

\usepackage{amssymb,mathrsfs}% Typical maths resource packages
\usepackage{amsthm}
\usepackage{bm}
\usepackage{scalerel}
\usepackage{nicefrac}
\usepackage{microtype} 
\usepackage[shortlabels]{enumitem}
\usepackage{graphicx}
\usepackage{epstopdf}
\DeclareGraphicsExtensions{.eps,.png,.jpg,.pdf}

\usepackage{url}
\usepackage{colortbl}
\usepackage{booktabs}
\usepackage{multirow}
\usepackage[table,dvipsnames]{xcolor}
\usepackage[normalem]{ulem}
\usepackage{xparse}
\usepackage{calc}
\usepackage{etoolbox}

\makeatletter
\@ifpackageloaded{natbib}{
	\relax
}{
	\usepackage{cite}
}
\makeatother

\usepackage{array}
\newcolumntype{L}[1]{>{\raggedright\let\newline\\\arraybackslash\hspace{0pt}}m{#1}}
\newcolumntype{C}[1]{>{\centering\let\newline\\\arraybackslash\hspace{0pt}}m{#1}}
\newcolumntype{R}[1]{>{\raggedleft\let\newline\\\arraybackslash\hspace{0pt}}m{#1}}

\makeatletter
\let\MYcaption\@makecaption
\makeatother
\usepackage[font=footnotesize]{subcaption}
\makeatletter
\let\@makecaption\MYcaption
\makeatother

\usepackage{glossaries}
\makeatletter
\sfcode`\.1006

\let\oldgls\gls
\let\oldglspl\glspl

\newcommand\fussy@ifnextchar[3]{%
	\let\reserved@d=#1%
	\def\reserved@a{#2}%
	\def\reserved@b{#3}%
	\futurelet\@let@token\fussy@ifnch}
\def\fussy@ifnch{%
	\ifx\@let@token\reserved@d
		\let\reserved@c\reserved@a
	\else
		\let\reserved@c\reserved@b
	\fi
	\reserved@c}

\renewcommand{\gls}[1]{%
\oldgls{#1}\fussy@ifnextchar.{\@checkperiod}{\@}}
\renewcommand{\glspl}[1]{%
\oldglspl{#1}\fussy@ifnextchar.{\@checkperiod}{\@}}

\newcommand{\@checkperiod}[1]{%
	\ifnum\sfcode`\.=\spacefactor\else#1\fi
}

\robustify{\gls}
\robustify{\glspl}
\makeatother

\newacronym{wrt}{w.r.t.}{with respect to}
\newacronym{RHS}{R.H.S.}{right-hand side}
\newacronym{LHS}{L.H.S.}{left-hand side}
\newacronym{iid}{i.i.d.}{independent and identically distributed}
\newacronym{SOTA}{SOTA}{state-of-the-art}
\usepackage{float}

\ifx\notloadhyperref\undefined
	\ifx\loadbibentry\undefined
		\usepackage[hidelinks,hypertexnames=false]{hyperref}
	\else
		\usepackage{bibentry}
		\makeatletter\let\saved@bibitem\@bibitem\makeatother
		\usepackage[hidelinks,hypertexnames=false]{hyperref}
		\makeatletter\let\@bibitem\saved@bibitem\makeatother
	\fi
\else
	\ifx\loadbibentry\undefined
		\relax
	\else
		\usepackage{bibentry}
	\fi
\fi

\usepackage[capitalize]{cleveref}
\crefname{equation}{}{}
\Crefname{equation}{}{}
\crefname{claim}{claim}{claims}
\crefname{step}{step}{steps}
\crefname{line}{line}{lines}
\crefname{condition}{condition}{conditions}
\crefname{dmath}{}{}
\crefname{dseries}{}{}
\crefname{dgroup}{}{}

\crefname{Problem}{Problem}{Problems}
\crefformat{Problem}{Problem~#2#1#3}
\crefrangeformat{Problem}{Problems~#3#1#4 to~#5#2#6}

\crefname{Theorem}{Theorem}{Theorems}
\crefname{Corollary}{Corollary}{Corollaries}
\crefname{Proposition}{Proposition}{Propositions}
\crefname{Lemma}{Lemma}{Lemmas}
\crefname{Definition}{Definition}{Definitions}
\crefname{Example}{Example}{Examples}
\crefname{Assumption}{Assumption}{Assumptions}
\crefname{Remark}{Remark}{Remarks}
\crefname{Rem}{Remark}{Remarks}
\crefname{remarks}{Remarks}{Remarks}
\crefname{Appendix}{Appendix}{Appendices}
\crefname{Supplement}{Supplement}{Supplements}
\crefname{Exercise}{Exercise}{Exercises}
\crefname{Theorem_A}{Theorem}{Theorems}
\crefname{Corollary_A}{Corollary}{Corollaries}
\crefname{Proposition_A}{Proposition}{Propositions}
\crefname{Lemma_A}{Lemma}{Lemmas}
\crefname{Definition_A}{Definition}{Definitions}

\usepackage{crossreftools}
\ifx\notloadhyperref\undefined
\else
	\relax
\fi

\usepackage{algorithm,algorithmic}

\ifx\loadbreqn\undefined
	\relax
\else
	\usepackage{breqn}
\fi

\makeatletter
\def\cleartheorem#1{%
    \expandafter\let\csname#1\endcsname\relax
    \expandafter\let\csname c@#1\endcsname\relax
}
\def\clearthms#1{ \@for\tname:=#1\do{\cleartheorem\tname} }
\makeatother

\ifx\renewtheorem\undefined
	\ifx\useTheoremCounter\undefined
		\newtheorem{Theorem}{Theorem}
		\newtheorem{Corollary}{Corollary}
		\newtheorem{Proposition}{Proposition}
		\newtheorem{Lemma}{Lemma}
	\else
		\newtheorem{Theorem}{Theorem}
		\newtheorem{Corollary}[Theorem]{Corollary}
		
	\fi

	\newtheorem{Definition}{Definition}

    \newtheorem{Problem}{Problem}

\fi

\theoremstyle{remark}

\theoremstyle{plain}

\newcommand{\qednew}{\nobreak \ifvmode \relax \else
		\ifdim\lastskip<1.5em \hskip-\lastskip
			\hskip1.5em plus0em minus0.5em \fi \nobreak
		\vrule height0.75em width0.5em depth0.25em\fi}



\NewDocumentCommand{\movedownsub}{e{^_}}{%
	\IfNoValueTF{#1}{%
		\IfNoValueF{#2}{^{}}% neither ^ nor _, do nothing; if no ^ but _, add ^{}
	}{%
		^{#1}% add superscript if present
	}%
	\IfNoValueF{#2}{_{#2}}% add subscript if present
}

\let\latexchi\chi
\RenewDocumentCommand{\chi}{}{\latexchi\movedownsub}

\newcommand{\calC}{\mathcal{C}}

\newcommand{\calE}{\mathcal{E}}

\newcommand{\calG}{\mathcal{G}}

\newcommand{\calN}{\mathcal{N}}
\newcommand{\calO}{\mathcal{O}}
\newcommand{\calP}{\mathcal{P}}

\newcommand{\calT}{\mathcal{T}}
\newcommand{\calU}{\mathcal{U}}
\newcommand{\calV}{\mathcal{V}}

\newcommand{\boldf}{\mathbf{f}}

\newcommand{\bg}{\mathbf{g}}

\newcommand{\bh}{\mathbf{h}}

\newcommand{\bx}{\mathbf{x}}

\newcommand{\frakR}{\mathfrak{R}}

\DeclareSymbolFont{bsfletters}{OT1}{cmss}{bx}{n}
\DeclareSymbolFont{ssfletters}{OT1}{cmss}{m}{n}
\DeclareMathSymbol{\bsfGamma}{0}{bsfletters}{'000}
\DeclareMathSymbol{\ssfGamma}{0}{ssfletters}{'000}
\DeclareMathSymbol{\bsfDelta}{0}{bsfletters}{'001}
\DeclareMathSymbol{\ssfDelta}{0}{ssfletters}{'001}
\DeclareMathSymbol{\bsfTheta}{0}{bsfletters}{'002}
\DeclareMathSymbol{\ssfTheta}{0}{ssfletters}{'002}
\DeclareMathSymbol{\bsfLambda}{0}{bsfletters}{'003}
\DeclareMathSymbol{\ssfLambda}{0}{ssfletters}{'003}
\DeclareMathSymbol{\bsfXi}{0}{bsfletters}{'004}
\DeclareMathSymbol{\ssfXi}{0}{ssfletters}{'004}
\DeclareMathSymbol{\bsfPi}{0}{bsfletters}{'005}
\DeclareMathSymbol{\ssfPi}{0}{ssfletters}{'005}
\DeclareMathSymbol{\bsfSigma}{0}{bsfletters}{'006}
\DeclareMathSymbol{\ssfSigma}{0}{ssfletters}{'006}
\DeclareMathSymbol{\bsfUpsilon}{0}{bsfletters}{'007}
\DeclareMathSymbol{\ssfUpsilon}{0}{ssfletters}{'007}
\DeclareMathSymbol{\bsfPhi}{0}{bsfletters}{'010}
\DeclareMathSymbol{\ssfPhi}{0}{ssfletters}{'010}
\DeclareMathSymbol{\bsfPsi}{0}{bsfletters}{'011}
\DeclareMathSymbol{\ssfPsi}{0}{ssfletters}{'011}
\DeclareMathSymbol{\bsfOmega}{0}{bsfletters}{'012}
\DeclareMathSymbol{\ssfOmega}{0}{ssfletters}{'012}

\newcommand{\bmu}{\bm{\mu}}

\makeatletter
\newcommand*\rel@kern[1]{\kern#1\dimexpr\macc@kerna}
\newcommand*\widebar[1]{%
  \begingroup
  \def\mathaccent##1##2{%
    \rel@kern{0.8}%
    \overline{\rel@kern{-0.8}\macc@nucleus\rel@kern{0.2}}%
    \rel@kern{-0.2}%
  }%
  \macc@depth\@ne
  \let\math@bgroup\@empty \let\math@egroup\macc@set@skewchar
  \mathsurround\z@ \frozen@everymath{\mathgroup\macc@group\relax}%
  \macc@set@skewchar\relax
  \let\mathaccentV\macc@nested@a
  \macc@nested@a\relax111{#1}%
  \endgroup
}
\makeatother

\DeclareMathOperator*{\argmin}{arg\,min}

\newcommand{\ifbcdot}[1]{\ifblank{#1}{\cdot}{#1}}

\DeclarePairedDelimiterX\abs[1]{\lvert}{\rvert}{\ifbcdot{#1}}
\DeclarePairedDelimiterX\parens[1]{(}{)}{\ifbcdot{#1}}
\DeclarePairedDelimiterX\brk[1]{[}{]}{\ifbcdot{#1}}
\DeclarePairedDelimiterX\braces[1]{\{}{\}}{\ifbcdot{#1}}
\DeclarePairedDelimiterX\angles[1]{\langle}{\rangle}{\ifblank{#1}{\cdot,\cdot}{#1}}
\DeclarePairedDelimiterX\ip[2]{\langle}{\rangle}{\ifbcdot{#1},\ifbcdot{#2}}
\DeclarePairedDelimiterX\norm[1]{\lVert}{\rVert}{\ifbcdot{#1}}
\DeclarePairedDelimiterX\ceil[1]{\lceil}{\rceil}{\ifbcdot{#1}}
\DeclarePairedDelimiterX\floor[1]{\lfloor}{\rfloor}{\ifbcdot{#1}}

\DeclareFontFamily{U}{matha}{\hyphenchar\font45}
\DeclareFontShape{U}{matha}{m}{n}{
      <5> <6> <7> <8> <9> <10> gen * matha
      <10.95> matha10 <12> <14.4> <17.28> <20.74> <24.88> matha12
      }{}
\DeclareSymbolFont{matha}{U}{matha}{m}{n}
\DeclareFontSubstitution{U}{matha}{m}{n}

\DeclareFontFamily{U}{mathx}{\hyphenchar\font45}
\DeclareFontShape{U}{mathx}{m}{n}{
      <5> <6> <7> <8> <9> <10>
      <10.95> <12> <14.4> <17.28> <20.74> <24.88>
      mathx10
      }{}
\DeclareSymbolFont{mathx}{U}{mathx}{m}{n}
\DeclareFontSubstitution{U}{mathx}{m}{n}

\DeclareMathDelimiter{\vvvert}{0}{matha}{"7E}{mathx}{"17}
\DeclarePairedDelimiterX\vertiii[1]{\vvvert}{\vvvert}{\ifbcdot{#1}}

\DeclarePairedDelimiterXPP\trace[1]{\operatorname{Tr}}{(}{)}{}{\ifbcdot{#1}} % column vector
\DeclarePairedDelimiterXPP\col[1]{\operatorname{col}}{\{}{\}}{}{\ifbcdot{#1}} % column vector
\DeclarePairedDelimiterXPP\row[1]{\operatorname{row}}{\{}{\}}{}{\ifbcdot{#1}} % row vector
\DeclarePairedDelimiterXPP\erf[1]{\operatorname{erf}}{(}{)}{}{\ifbcdot{#1}}
\DeclarePairedDelimiterXPP\erfc[1]{\operatorname{erfc}}{(}{)}{}{\ifbcdot{#1}}
\DeclarePairedDelimiterXPP\KLD[2]{D}{(}{)}{}{\ifbcdot{#1}\, \delimsize\|\, \ifbcdot{#2}} % KL divergence
\DeclarePairedDelimiterXPP\op[2]{\operatorname{#1}}{(}{)}{}{#2} % general operator

\newcommand{\T}{^{\mkern-1.5mu\mathop\intercal}}% transpose notation
\newcommand{\ud}{\,\mathrm{d}} % for integrals like \int f(x) \ud x
\DeclarePairedDelimiterXPP\indicate[1]{{\bf 1}}{\{}{\}}{}{\ifbcdot{#1}}

\NewDocumentCommand\ofrac{s m}{%
	\IfBooleanTF#1%
	{\dfrac{1}{#2}}%
	{\frac{1}{#2}}%
}
\NewDocumentCommand\ddfrac{s m m}{%
	\IfBooleanTF#1%
	{\dfrac{\mathrm{d} {#2}}{\mathrm{d} {#3}}}%
	{\frac{\mathrm{d} {#2}}{\mathrm{d} {#3}}}%
}
\NewDocumentCommand\ppfrac{s m m}{%
	\IfBooleanTF#1%
	{\dfrac{\partial {#2}}{\partial {#3}}}%
	{\frac{\partial {#2}}{\partial {#3}}}%
}

\providecommand\given{}
\DeclarePairedDelimiterX\Set[2]\{\}{%
\renewcommand\given{\SetSymbol[\delimsize]{#1}}
#2
}
\DeclarePairedDelimiterX\Setc[1]\{\}{%
\renewcommand\given{\SetSymbol{:}}
#1
}

\NewDocumentCommand\set{s o m}{%
	\IfBooleanTF#1%
	{\IfValueTF{#2}{\Set*{#2}{#3}}{\Setc*{#3}}}%
	{\IfValueTF{#2}{\Set{#2}{#3}}{\Setc{#3}}}%
}

\NewDocumentCommand{\evalat}{ s O{\big} m e{_^} }{%
\IfBooleanTF{#1}%
{\left. #3 \right|}{#3#2|}%
\IfValueT{#4}{_{#4}}%
\IfValueT{#5}{^{#5}}%
}

\providecommand\given{}
\DeclarePairedDelimiterXPP\cprob[1]{}(){}{
\renewcommand\given{\nonscript\,\delimsize\vert\allowbreak\nonscript\,\mathopen{}}%
#1%
}
\DeclarePairedDelimiterXPP\cexp[1]{}[]{}{
\renewcommand\given{\nonscript\,\delimsize\vert\allowbreak\nonscript\,\mathopen{}}%
#1%
}

\DeclareDocumentCommand \P { s e{_^} d() g } {%
	\mathbb{P}%
	\IfBooleanTF{#1}%
		{
			\IfValueT{#2}{_{#2}}%
			\IfValueT{#3}{^{#3}}%
			\IfValueTF{#5}{\cprob{#4 \given #5}}{\IfValueT{#4}{\cprob{#4}}}%
		}%
		{
			\IfValueT{#2}{_{#2}}%
			\IfValueT{#3}{^{#3}}%
			\IfValueTF{#5}{\cprob*{#4 \given #5}}{\IfValueT{#4}{\cprob*{#4}}}%
		}%
}

\DeclareDocumentCommand \E { s e{_^} o g } {%
	\mathbb{E}%
	\IfBooleanTF{#1}%
		{
			\IfValueT{#2}{_{#2}}%
			\IfValueT{#3}{^{#3}}%
			\IfValueTF{#5}{\cexp{#4 \given #5}}{\IfValueT{#4}{\cexp{#4}}}%
		}%
		{
			\IfValueT{#2}{_{#2}}%
			\IfValueT{#3}{^{#3}}%	
			\IfValueTF{#5}{\cexp*{#4 \given #5}}{\IfValueT{#4}{\cexp*{#4}}}%		
		}%
}

\ExplSyntaxOn
\NewDocumentCommand \dist {m o o} {%
\mathrm{#1}\left(%
	\IfValueT{#3}{%
		\tl_if_blank:nTF{ #3 }{\cdot\, \middle|\, }{#3\, \middle|\, }%
	}
	\IfValueT{#2}{#2}%
\right)%
}
\ExplSyntaxOff

\NewDocumentCommand {\cbrace} {t+ D[]{black} D(){\widthof{#5}} m m } {%
	\begingroup%
		\color{#2}
		\IfBooleanTF{#1}{%
			\overbrace{#4}^%
		}{
			\underbrace{#4}_%
		}%
		{\parbox[c]{#3}{\centering\footnotesize{#5}}}%
	\endgroup% 
}

\let\oldforall\forall
\renewcommand{\forall}{\oldforall \, }

\let\oldexist\exists
\renewcommand{\exists}{\oldexist \, }

\newcommand{\figref}[1]{Fig.~\ref{#1}}
\graphicspath{{./Figures/}{./figures/}}
\DeclareDocumentCommand{\includeCroppedPdf}{ o O{./Figures/} m }{
	\IfFileExists{#2#3-crop.pdf}{}{%
		\immediate\write18{pdfcrop #2#3.pdf #2#3-crop.pdf}}%
	\includegraphics[#1]{#2#3-crop.pdf}
}

\makeatletter
\newcommand*{\addFileDependency}[1]{% argument=file name and extension
  \typeout{(#1)}
  \@addtofilelist{#1}
  \IfFileExists{#1}{}{\typeout{No file #1.}}
}
\makeatother

\definecolor{gray90}{gray}{0.9}

\ifx\nohighlights\undefined

	\newcommand{\msout}[1]{\text{\color{green} \sout{\ensuremath{#1}}}}
	\newcommand{\del}[1]{{\color{green}\ifmmode \msout{#1}\else\sout{#1}\fi}}
\else

	\newcommand{\msout}[1]{#1}
	\newcommand{\del}[1]{#1}
\fi

\newcommand{\hhide}[1]{}
\ifx\diagnoselabel\undefined
	\relax
\else
	\makeatletter
	\def\@testdef #1#2#3{%
		\def\reserved@a{#3}\expandafter \ifx \csname #1@#2\endcsname
			\reserved@a  \else
			\typeout{^^Jlabel #2 changed:^^J%
				\meaning\reserved@a^^J%
				\expandafter\meaning\csname #1@#2\endcsname^^J}%
			\@tempswatrue \fi}
	\makeatother
\fi

\newtheorem{Method}{Method}

\usepackage{microtype}
\usepackage{graphicx}
\usepackage{booktabs} % for professional tables

\usepackage{hyperref}

\usepackage[accepted]{icml2023}

\usepackage{amsmath}
\usepackage{amssymb}
\usepackage{mathtools}
\usepackage{amsthm}
\usepackage{arydshln}

\usepackage[textsize=tiny]{todonotes}

\icmltitlerunning{Leveraging Label Non-Uniformity for NC in GNNs}

\begin{document}

\twocolumn[
\icmltitle{Leveraging Label Non-Uniformity for Node Classification in Graph Neural Networks}




\begin{icmlauthorlist}
\icmlauthor{Feng Ji}{ntu}
\icmlauthor{See Hian Lee}{ntu}
\icmlauthor{Hanyang Meng}{jn}
\icmlauthor{Kai Zhao}{ntu}
\icmlauthor{Jielong Yang}{jn}
\icmlauthor{Wee Peng Tay}{ntu}
%\icmlauthor{Firstname7 Lastname7}{comp}
%\icmlauthor{Firstname8 Lastname8}{sch}
%\icmlauthor{Firstname8 Lastname8}{yyy,comp}
\end{icmlauthorlist}

\icmlaffiliation{ntu}{School of Electrical and Electronic Engineering, Nanyang Technological University, Singapore}
\icmlaffiliation{jn}{School of Internet of Things Engineering, Jiangnan University, Wuxi, Jiangsu, China}

\icmlcorrespondingauthor{J. Yang}{jyang022@e.ntu.edu.sg}
\icmlcorrespondingauthor{W. P. Tay}{wptay@ntu.edu.sg}

\icmlkeywords{Machine Learning, ICML}

\vskip 0.3in
]



\printAffiliationsAndNotice{}  % leave blank if no need to mention equal contribution

\begin{abstract}
In node classification using graph neural networks (GNNs), a typical model generates logits for different class labels at each node. A softmax layer often outputs a label prediction based on the largest logit. We demonstrate that it is possible to infer hidden graph structural information from the dataset using these logits. We introduce the key notion of label non-uniformity, which is derived from the Wasserstein distance between the softmax distribution of the logits and the uniform distribution. We demonstrate that nodes with small label non-uniformity are harder to classify correctly. We theoretically analyze how the label non-uniformity varies across the graph, which provides insights into boosting the model performance: increasing training samples with high non-uniformity or dropping edges to reduce the maximal cut size of the node set of small non-uniformity. These mechanisms can be easily added to a base GNN model. Experimental results demonstrate that our approach improves the performance of many benchmark base models.   
\end{abstract}

\section{Introduction}
Graph neural networks (GNNs) are neural networks that learn from graph-structured data \citep{Def16}. A problem of interest is the node classification problem. In a graph, one is given the labels of a subset of nodes and must predict the labels of remaining nodes using features associated with each node. Through many years of development, numerous GNN models have been proposed to tackle the node classification problem. Though many more recent approaches (e.g., \citet{Zha20, Zhu20, Yan21, Kang21, Rus22, Song22, Zhao23}) have sophisticated mechanisms, they are influenced by earlier models such as GCN \citep{Def16} and GAT \citep{Vel18}. Quite a few important features of these primitive models are inherited by their up-to-date counterparts.

We briefly recall a few features of a model such as GCN that are most relevant. The basic idea of a GNN model is to generate an embedding of nodes in an appropriate geodesic metric space such as Euclidean space. Nodes are subsequently grouped into different classes using a union of hyperplanes. Such a strategy is realized by applying a message-passing algorithm. In each iteration of the algorithm, each node updates its feature vector by using a weighted average of the feature vectors from its neighbors. The generated embedding is then input to fully connected layers that output logits, which after normalization via softmax, are interpreted as probability weights of the label classes. The predicted label of each node is the class with the largest probability. Geometric information (e.g., hyperbolicity \citet{Gul19, Zhu20, Zhay21}) on the graph plays a fundamental role in the process as we need to utilize the local neighborhood of each node. 

On the other hand, there is hidden graph structural information crucial to the success of the classification task. For example, we may be interested to know what are the boundary nodes (i.e., those nodes that have neighbors with different classes) between different label classes. In this work, we aim to retrieve such graph structural information using predicted logits from a model such as GCN. Therefore, in contrast to the procedure described in the previous paragraph, information retrieval is in the reverse direction. 

The main tool we use is the notion of non-uniformity of the probability weights of the label classes derived from logits, inspired by the idea of distributional graph signals \citep{Jif23, Ji23}. It measures the extent to which the probability distribution is not uniform.
The insight is that for a node near class boundaries, during training, we have mixed contributions 
from different label classes as some of its neighbors belong to different classes. As message-passing has a smoothing effect \citep{Oon20}, the resulting predicted logits should reflect the phenomenon that several weights for different classes can have similar values. Therefore, non-uniformity may reveal hidden graph structural information associated with the embedding of different node classes in the ambient graph. %\red{[You don't want to confuse readers who may be thinking about the graph structural spaces like hyperbolic space embeddings. It seems a more appropriate term to use in this paper is ``graph topological information'' or ``graph structural information''.]}. 
In this paper, we theoretically study how the above-mentioned notion of non-uniformity provides us with graph structural knowledge (e.g., whether a node is close in graph distance to a class boundary).  %\red{[Re-read this sentence, I can't get its meaning.]}
Based on the findings, we propose a simple multi-step model with independent modulo components, whose effectiveness is demonstrated with numerical experiments. Proofs of theoretical results are provided in \cref{sec:pro}.
Our main contributions are summarized as follows:
\begin{itemize}
    \item We introduce the notion of label non-uniformity of probability weights associated with label classes. We demonstrate experimentally that nodes with small non-uniformity are harder to classify correctly.
    \item We analyze locations (with respect to class boundaries) of nodes with small label non-uniformity and provide insights to increase their non-uniformity. 
    \item We propose two algorithms to boost the performance of a given base model: increasing training samples with high predicted label non-uniformity or dropping edges to reduce the maximal cut size of a node set of small non-uniformity. Experiments indicate that by adding our algorithm modules to a base model, its performance can be improved. 
\end{itemize} 

\section{Label non-uniformity and node selection} \label{sec:uni}

Suppose $\calG=(\calV,\calE)$ is an undirected graph with $\calV$ the vertex set and $\calE$ the edge set. Let $d_{\calG}$ be the metric with $d_{\calG}(v,v')$ being the length of the shortest path between $v,v' \in \calV$. 

We consider the node classification problem. Recall that there is a finite set $\mathbb{S}$ of class labels. Each $v\in \calV$ has a label $s\in \mathbb{S}$. For each $s\in \mathbb{S}$, let $\calV_s$ be the set of nodes with class label $s$. We want to train a model based on a training set $\mathfrak{R}$ to predict the (unknown) labels of a test set $\mathfrak{T}$. To motivate the concepts we want to bring about, consider a variant of node classification as follows, which is interesting in its own right. We point out that the problem is a hypothetical thought experiment, but not the main subject of the paper. The purpose of the study is to investigate empirical evidence of the relations among easily classifiable nodes, the predicted logits, and graph topology.
\begin{Problem}\label{prob:fao}
We use $\mathfrak{R}$ to train a GCN. For given $\alpha \leq 1$, one is required to choose a subset $\mathfrak{T}' \subset \mathfrak{T}$ of size $\alpha |\mathfrak{T}|$ so that the test accuracy of the trained model on $\mathfrak{T}'$ is maximized. 
\end{Problem}

As a special case, $\alpha=1$ is the original node classification problem. The modified problem essentially asks one to find a subset of $\mathfrak{T}$ with a prescribed size, on which one can make an accurate prediction. Intuitively, suppose $v$ has label $s$. Then it is more likely to predict correctly for $v$ if it is close in distance to training nodes in $\calV_s$. This prompts:    

\begin{Method}[M1: the geometric method]
For each node $v\in \mathfrak{T}$, we assgin with it a pair of numbers $(f_1(v),f_2(v))$, where $f_1(v) = \min_{v'\in \mathfrak{R}}d_{\calG}(v,v')$. For $f_2$, we first find the set $\argmin_{v'\in \mathfrak{R}}d_{\calG}(v,v')$. Then let $g(v)$ be the percentage of the largest label class in $\argmin_{v'\in \mathfrak{R}}d_{\calG}(v,v')$ and $f_2(v) = 1-g(v)$. We rank $\mathfrak{T}$ based on the lexographic order of $\big(f_1(v),f_2(v)\big)$ and $\mathfrak{T}'$ is chosen following the ordering. The second number $f_2(v)$ is the tie-breaker for nodes with the same $f_1$ value.  
\end{Method}

\begin{figure}
    \centering \includegraphics[width=0.83\columnwidth]{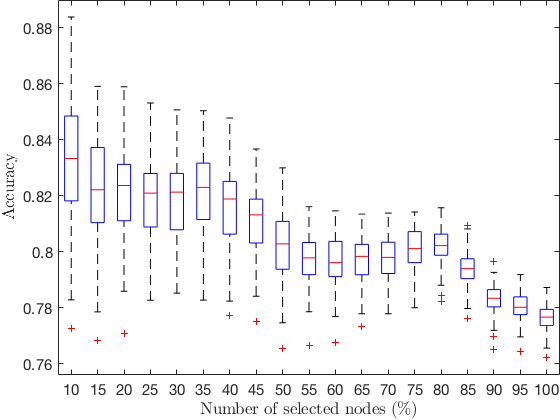} \includegraphics[width=0.83\columnwidth]{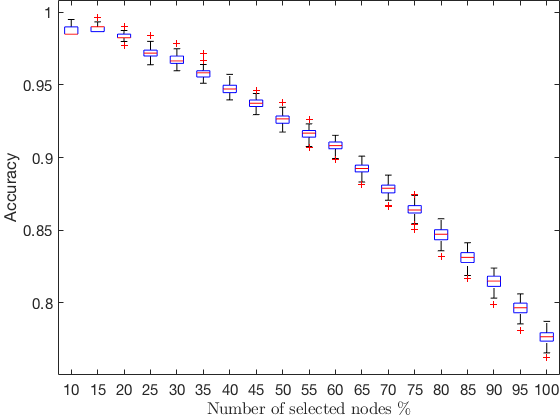}
 \caption{Accuracy for M1, M2 respectively on the Cora dataset.}\label{fig:rank}
 \end{figure}

We run multiple experiments on the Cora dataset and the results are shown in \cref{fig:rank} (top). We see that the performance pattern as $\alpha$ increases indeed supports the geometric intuition. However, we may push the intuition further to enhance the performance. We first introduce the notion of the boundary of a vertex set. Let $\calV'\subset \calV$ be a subset of vertices. Its \emph{boundary} $\partial \calV'$ is: 
\begin{align*}
    \partial \calV' = \{v\in \calV' \mid \exists v'\notin \calV' \text{ s.t.\ } (v,v')\in \calE\}.
\end{align*}
Another insight on graph structure we may leverage is that when we want to determine the class label $s$ of $v\in \calV_s$, it is likely that we are less certain if $v$ is closer to the boundary $\partial \calV_s$. This is because there are nodes with different class labels nearby (in $d_{\calG}$). %\red{[Why is this an ``insight'' and how to ``leverage'' it?]}

An immediate challenge to exploit the above insight is that before knowing the true labels of all the nodes, $\partial \calV_s$ is usually obscure to us. Hence, we need to find a way to estimate: \emph{for any given $v$ with label $s$, how far away it is from $\partial \calV_s$ without knowing the label $s$.} For this, we make use of the concept of non-uniformity associated with label distribution. 

Recall that from the logits of the base model trained on $\mathfrak{R}$, we may apply softmax to obtain a vector of numbers $\bmu_v$ with $\bmu_v(s)\in [0,1], s\in \mathbb{S}$ for each $v\in \calV$. Moreover, $\mathbb{S}$ is a finite set and $\sum_{s\in \mathbb{S}}\bmu_v(s)=1$, therefore $\bmu_v(s)$ can be interpreted as the probability weight of node $v$ having label $s$. Following \citet{Ji23},
the \emph{label non-uniformity} at $v$ is defined by 
\begin{align} \label{eq:wsi}
    w(v) = \sum_{s\in \mathbb{S}} \abs*{\bmu_v(s) - \frac{1}{|\mathbb{S}|}}.
\end{align}
The notion is derived from the $2$-Wasserstein distance \citep{Vil09, Ji23} between $\bmu_v$ and the uniform distribution. We justify in \cref{sec:pro}. The function $w(\cdot)$ is the key player of our approach suggested by our title.

We propose to use $w(v)$ to measure whether $v$ is close in distance to some class boundary. Nodes closer to class boundaries are expected to be harder to classify. This prompts the following approach to \cref{prob:fao}. 

\begin{Method}[M2: distribution non-uniformity]
In training, we obtain a probability distribution $\bmu_v$ on $\mathbb{S}$ for each $v\in \calV$ (in the last layer). We rank $\mathfrak{T}$ according to non-uniformity $w(v)$ and $\mathfrak{T}'$ is chosen following the ordering. 
\end{Method}

The results for the approach M2 are also shown in \figref{fig:rank} (bottom). It has a much better performance, which is also more consistent. In the following, we summarize observations from the experiments that eventually lead to our proposed GNN model.
\begin{enumerate}[(a)]
    \item Using the non-uniformity $w(\cdot)$, we can reasonably identify nodes whose labels are correctly predicted.
    \item Comparing M1 and M2, we notice the correctly labeled nodes are not necessarily close to training nodes in $\mathfrak{R}$.
\end{enumerate}

In the next section, we theoretically justify our graph structural intuition regarding the function of label non-uniformity $w(\cdot)$. Moreover, we propose a new GNN model based on the theoretical findings. 

\section{Label non-uniformity and graph structural information} \label{sec:gta}

As we speculate in the previous section, we want to analyze the non-uniformity $w(v)$ in view of the structural information of $\calG$ in this section. Intuitively, if node $v$ has large non-uniformity $w(v)$, its distribution weights $\bmu_v(s), s\in \mathbb{S}$ are either close to $0$ or $1$. As $w(\cdot)$ is computed from $\bmu_v$, to study $w(\cdot)$, we may instead analyze $\bmu_v$ in this section. 

\subsection{Flow of probability weights}
For an overview, we first understand the flow of probability weights from a subset of nodes to another by analyzing the solution of an optimization problem. The study allows us to acquire geometric information such as the location of class boundaries using the weights. Furthermore, we analyze how to create a bottleneck near class boundaries to widen the difference in probability weights for nodes near class boundaries with different labels.

For the graph $\calG=(\calV,\calE)$, let $L_{\calG}$ be the Laplacian of $\calG$. To simplify the analysis, we study one class label at a time. Fix a class label $s\in \mathbb{S}$ and consider the \emph{graph signal of probability weights} $(\bmu_{v}(s))_{v\in \calV}$. We expect a model to make an accurate fitting on the training set $\mathfrak{R}$, therefore it is reasonable to assume that $\bmu_{v}(s) \approx 1, v\in \mathfrak{R}\cap \calV_s$ and $\bmu_{v'}(s) \approx 0, v'\in \mathfrak{R}\cap \calV_{s'}, s'\neq s$.

With this in mind, we assume that there are disjoint non-empty subsets $\calO_0$ and $\calO_1$ of $\calV$ and an approximation $\boldf$ of $(\bmu_{v}(s))_{v\in \calV}$ such that the graph signal $\boldf$ is observed at $\calO = \calO_0\cup \calO_1$.\footnote{To avoid cluttered notations, we omit $s$ from the symbol $\boldf$.} Moreover, $\boldf_v = 0, v \in \calO_0$ and $\boldf_{v'} = 1, v'\in \calO_1$. The primary example is $\calO_1 \subset \mathfrak{R}\cap \calV_s$ and $\calO_0 \subset \mathfrak{R}\backslash \calV_s$. A \emph{smooth interpolation} $\widetilde{\boldf}$ of $\boldf$ is \begin{align*}
    \widetilde{\boldf} = \argmin_{\boldf': \boldf'_v = \boldf_v,v\in \calO} {\boldf'}\T L_{\calG}{\boldf'}.
\end{align*} 
This is the well-studied Laplacian quadratic form, and its use in interpolating graph signals is justified in \citet{Zhu03, Zho04, And07, Shu13, Nar13}. We shall justify (in \cref{sec:pro}) that if $\boldf$ is a good approximation of $(\bmu_{v}(s))_{v\in \calV}$, then they have similar smooth interpolations. We use $\widetilde{\boldf}$ as a proxy of the true $\boldf$ and hence $(\bmu_{v}(s))_{v\in \calV}$ on the entire graph, based on the assumption that $\boldf$ and $(\bmu_{v}(s))_{v\in \calV}$ are smooth. Therefore, we have reduced the study of $(\bmu_{v}(s))_{v\in \calV}$ to that of $\widetilde{\boldf}$, and the fundamental result is the following \emph{averaging property}. %\red{[$\E_{(v,v')\in \calE}$ has not been defined.]}

\begin{Lemma} \label{lem:fev}
    For every $v\notin \calO$, let $\mathfrak{d}_v$ be its degree. We have
    \begin{align}
        \widetilde{\boldf}_v = \frac{1}{\mathfrak{d}_v}\sum_{(v,v')\in \calE}\widetilde{\boldf}_{v'}.
    \end{align}
\end{Lemma}

As a consequence, we can view $\calO_0$ and $\calO_1$ analogous to the poles in a magnet, illustrated in \figref{fig:flow} (a). To be more precise, for each $v\in \calV$, define the \emph{level-component $\calC_v$} of $v$ w.r.t.\ $\widetilde{\boldf}$ to be the connected component of $\{v'\in \calV\mid \widetilde{\boldf}_{v'} = \widetilde{\boldf}_v\}$ containing $v$. Then the following holds. 

\begin{Theorem} \label{prop:fev}
    For each $v\in \calV$, there is a path $\calP=\{v_0,\ldots,v_m\}$ such that the following holds:
    \begin{enumerate}[(a)]
        \item $v_0\in \calO_0$ and $v_m\in \calO_1$.
        \item $\widetilde{\boldf}$ is strictly increasing on $\calP$: $\widetilde{\boldf}_{v_i} < \widetilde{\boldf}_{v_{i+1}}, 1\leq i<m$.
        \item $\calP\cap \calC_v \neq \emptyset$. 
    \end{enumerate}
\end{Theorem}

The definition of a level-component is for the technical issue that there might be neighboring nodes with the same $\widetilde{\boldf}$ value. Disregarding this technicality, intuitively, \cref{prop:fev} claims that $\widetilde{\boldf}$ values at nodes close to $\calO_0$ should be close to $0$ and gradually increase to $1$ along the path $\calP$. The following corollary of the theorem describes the signal pattern emanating from $\calO_1$ as illustrated in \figref{fig:flow}(b).

\begin{figure}[!htb]
    \centering
    \includegraphics[scale=0.5]{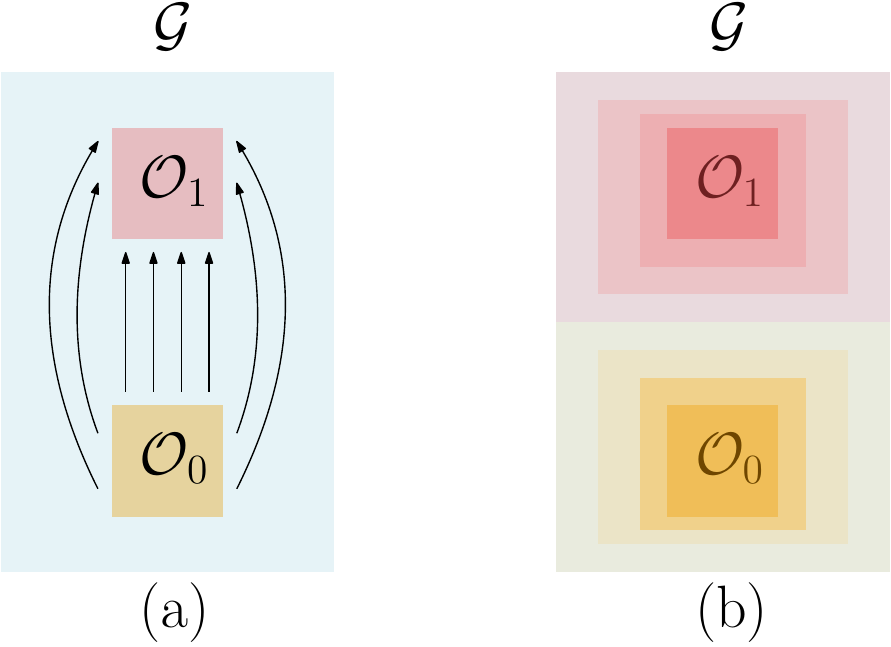}
    \caption{The Venn diagrams illustrate \cref{prop:fev} and \cref{coro:fau}. In (a), the arrows are paths along which $\widetilde{\boldf}$ increases. In (b), if we color the graph according to $\widetilde{\boldf}$ value, then it should be layered as shown.}
    \label{fig:flow}
\end{figure}

\begin{Corollary} \label{coro:fau}
Suppose for any $u,v\notin \calO_0\cup \calO_1$, we have $\widetilde{\boldf}_u \neq \widetilde{\boldf}_v$. For any $0<r<1$, define $\calV_r=\{v\in \calV \mid \widetilde{\boldf}_v \leq r\}$ and let $\calG_{\calV_r}$ be its induced subgraph. Denote the complement of $\calV_r$ by $\calV_r^c$. If $\calG_{\calV_{r_0}}$ is connected for some $r_0$, then so is $\calG_{\calV_r}$ for any $r\geq r_0$. Similarly, if $\calG_{\calV_{r_0}^c}$ is connected for some $r_0$, then so is $\calG_{\calV_r^c}$ for any $r\leq r_0$.
\end{Corollary}

Intuitively, the corollary describes the picture that if $\mathcal{G}_{\mathcal{V}_{r_0}}$ is connected, then $\mathcal{G}_{\mathcal{V}_{r}}$ must ``grow'' from $\mathcal{G}_{\mathcal{V}_{r_0}}$ for any $r\geq r_0$. We have seen an overall pattern of $\widetilde{\boldf}$ on $\calG$. Next, we study more refined details of its values along a set boundary.

\subsection{Weights near a boundary} \label{sec:wei}

To gain further insights, we describe another consequence of the averaging property in this subsection. For any subset $\calV_0\subset \calV$, in addition to its boundary $\partial \calV_0$, we define $\Gamma(\calV_0) \subset \calE$ to be the set of edges $(v,v')$ with $v\in \calV_0$ and $v'\in \calV_1$, where $\calV_1 = \calV\backslash \calV_0$. Moreover, we introduce the quantity %\red{[Isn't the sum over the newly defined $\Gamma(\calV_0)$? Why is it not used?]}
\begin{align*}
    A(\widetilde{\boldf},\calV_0) = \sum_{\substack{v\in \partial \calV_0,\\ (v,v')\in \Gamma(\calV_0)}} \widetilde{\boldf}_v
\end{align*}
as a weighted sum of $\widetilde{\boldf}_v, v\in \partial \calV_0$.

\begin{Theorem} \label{cor:lvb}
    Suppose $\calV_0$ and $\calV_1 = \calV\backslash \calV_0$ are disjoint subsets of $\calV$ such that $\calO_i$ is contained in the interior $\calV_i\backslash \partial \calV_i$ of $\calV_i, i=0,1$. %\red{[$\calO_i = \calV_i\backslash \partial\calV_i$? Why are $\calV_0$ and $\calV_1$ necessarily disjoint?]} 
    Assume that $0<\widetilde{\boldf}_v<1$ for $v\notin \calO$ and let $a = \min_{v\notin \calO} \widetilde{\boldf}_v$ and $b = \min_{v\notin \calO} \{1-\widetilde{\boldf}_v\}$. Then 
    \begin{align} \label{eq:awc}
    A(\widetilde{\boldf},\calV_0) \leq  A(\widetilde{\boldf},\calV_1) - \max(b|\Gamma(\calO_1)|, a|\Gamma(\calO_0)|).
    \end{align}
\end{Theorem}

\Cref{cor:lvb} claims that if the graph is partitioned into two parts $\calV_0, \calV_1$ (e.g., $\calV_1$ consists of nodes with class label $s$, where $s$ is as in the first paragraph of this section) containing $\calO_0$ and $\calO_1$ respectively, %\red{[Unsure of this interpretation. The definition of $A$ is not average.]}
then the weighted sum of the values of $\widetilde{\boldf}$ along the boundary on the $\calO_0$ side are smaller than those along the boundary on the $\calO_1$ side. To find the average difference, it suffices to divide both sides of (\ref{eq:awc}) by $|\Gamma(\calV_0)|$, as both $A(\widetilde{\boldf},\calV_0)$ and $A(\widetilde{\boldf},\calV_1)$ have $|\Gamma(\calV_0)|$ terms in their respective summation. Increasing the average $\widetilde{\boldf}$ difference is exactly what we are aiming for from \cref{sec:uni}: reduce the number of nodes with small non-uniformity. 

\begin{figure}[!htb]
    \centering
    \includegraphics[scale=0.5]{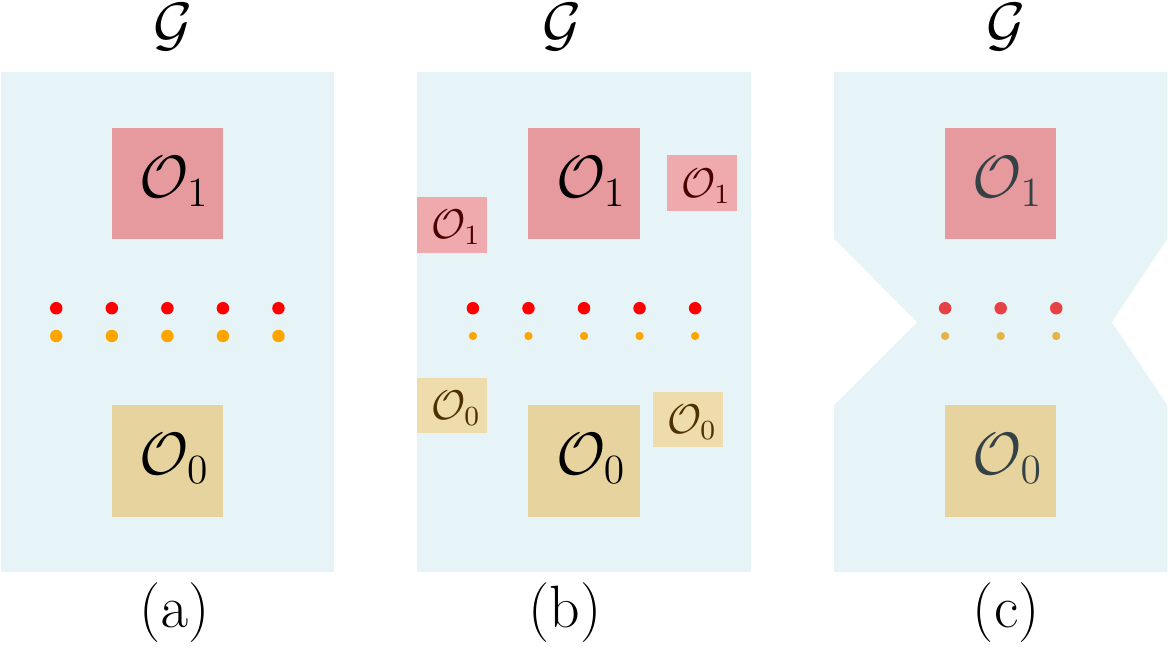}
    \caption{The Venn diagram illustrates that to differentiate nodes using $\widetilde{\boldf}$ value along the boundary, we may either increase the size of $\calO_0$ and $\calO_1$ as in (b) or create a bottleneck as in (c).}
    \label{fig:add}
\end{figure}

From \cref{prop:fev} and \cref{cor:lvb}, we have an overall picture:
\begin{enumerate}[(a)]
    \item By \cref{prop:fev}, the values of $\widetilde{\boldf}$ gradually increase from nodes near $\calO_0$ to those near $\calO_1$. The boundary should be more likely to occur at nodes with $\widetilde{\boldf}$ further away from $0$ and $1$. 
    \item By \cref{cor:lvb}, to make $\widetilde{\boldf}$ closer to $0$ or $1$ on average even near class boundaries, we have two options (illustrated in \figref{fig:add}): (i) enlarge $\calO_1$ and $\calO_0$, or (ii) reduce the size of $\Gamma(\calV_i), i=0,1$, i.e., we want to create a bottleneck. We discuss the second option further in the next subsection. 
\end{enumerate}

\subsection{Graph bottleneck} \label{sec:gbo}

The \emph{Cheeger constant} or \emph{edge expansion} \citep{Moh89} $h(\calG)$ is usually used to measure the ``bottleneck size'' of a connected graph $\calG$:
\begin{align*}
    h(\calG) = \min_{\calV'\subset \calV} \frac{|\Gamma(\calV')|}{\min(|\calV'|,|{\calV'}^c|)}.
\end{align*}
Other related concepts are the \emph{max-cut size} $C(\calG)$ and \emph{min-cut size} $c(\calG)$. Recall that a \emph{cut} is $\Gamma(\calV')$ for $\emptyset\neq \calV'\subset \calV$ such that ${\calV'}^{c}\neq \emptyset$. Then $C(\calG)$ (resp.\ $c(\calG)$) is the largest (resp.\ smallest) among the size of every cut of $\calG$. 

For any $\calV' \subset \calV$, let $\calG_{\calV'}$ be its induced subgraph. We say a numerical invariant is \emph{independent of $\calG_{\calV'}$} if it does not change, even if the edge set of $\calG_{\calV'}$ is modified. As we shall see, the proposed method requires one to modify the edge set of $\calG_{\calV'}$. Therefore, this notion of independence is important for theoretical results.  

\begin{Theorem} \label{thm:fcc}
For $\calV'\subset \calV$, suppose $\calG_{{\calV'}^c}$ has $2$-connected components $\calU_0$, $\calU_1$, whose respective neighbors in $\calV'$ are disjoint. Then there are constants $c_0,c_1 >0$ independent of $\calG_{\calV'}$ such that: if 
    $C(\calG_{\calV'}) < c_0$, then $h(\calG) \leq C(\calG_{\calV'})/c_1$.
Moreover, if $\calU \subset \calV$ realizes $h(\calG)$, then $\calU\cap \calV' \neq \emptyset$ and $\calU^c\cap \calV' \neq \emptyset$. 
\end{Theorem}
The constants $c_0$ and $c_1$ are made explicit in the proof (cf.\ Appendix~\ref{sec:pro}). In particular, they are related to the cut size of $\calU_0$ and $\calU_1$, as expected. The appendix contains additional discussions and illustrations. 

\Cref{thm:fcc} essentially says that to create a bottleneck of the graph, we may reduce the cut size of a separating subgraph, i.e., a subgraph that separates the ambient graph into two connected components. This subgraph consists of nodes near class boundaries in our setup. This idea inspires \cref{algo:edu} in the next section.

\section{Utilizing label non-uniformity}

%\red{[In previous section, the impression is that we are solving Problem 1. There is no discussion on ``translating'' lessons to the node classification problem. In this section, is the GNN for solving Problem 1?]}
In \cref{sec:uni}, we use hypothetical experiments on \cref{prob:fao} to motivate the study of non-uniformity $w(\cdot)$ in \cref{sec:gta}. Our experiments on \cref{prob:fao} suggest that nodes with large non-uniformity are those that we can classify more accurately. Based on the theoretical insights derived in \cref{sec:gta}, we now propose a model that increases the number of nodes with large non-uniformity in the training set, which leads to a boost in the model performance in testing.

As alluded to in the last paragraph of \cref{sec:wei}, we can achieve our goal by (i) introducing more nodes in the training set with accurate labels (cf.\ \figref{fig:add}(b)), or (ii) creating a bottleneck near class boundaries (cf.\ \figref{fig:add}(c)). Both are associated with the non-uniformity function $w(\cdot)$ in (\ref{eq:wsi}). However, $w(\cdot)$ is used in different ways in these two approaches. In view of the experimental and theoretical findings in the previous sections, for (i), we prioritize nodes with large non-uniformity, while for (ii), we consider nodes with smaller non-uniformity as candidates for boundary nodes. The algorithms are presented in \cref{algo:pab,algo:edu}. 
%\red{[Please explain the intuition/logic behind these algorithms, especially Algo 2.]}

\begin{algorithm}[!htb] 
\caption{Supplement the training set $\mathfrak{R}$ using $w(\cdot)$} \label{algo:pab}
\begin{enumerate}[(a)]
\item \label{it:pab} Pick a base GNN model $\mathfrak{M}$ (e.g., GCN, GAT) and train the model $\mathfrak{M}$ to obtain the label class probability $\bmu_v(s)$, where $s\in \mathbb{S}$ are the label classes, for each node $v\in \calV$. Compute $w(v) = \sum_{s\in \mathbb{S}} |\bmu_v(s) - \frac{1}{|\mathbb{S}|}|$.
\item\label{it:order} Order the test nodes $v$ in $\mathfrak{T}$ in decreasing order according to $w(v)$.
\item For a hyperparameter $\eta_0$, we form $\calV'$ by taking $\eta_0$ fraction of nodes with the largest $w(\cdot)$ values, i.e., nodes higher in the ordering above.
\item\label{it:R'} Nodes in $\calV'$ are added to the training set $\mathfrak{R}$ with their \emph{predicted test labels} in \ref{it:pab} to form $\mathfrak{R}'$, so that no ground-truth information is leaked.
\item Retrain $\mathfrak{M}$ with $\mathfrak{R}'$
\end{enumerate}
\end{algorithm}

\begin{algorithm}[!htb]
\caption{Edge dropping using $w(\cdot)$} \label{algo:edu}
    \begin{enumerate}[(a)]
    \item Same as \cref{algo:pab}\ref{it:pab} and \ref{it:order}.
    \item For a hyperparameter $\eta_1$, we form $\calV'$ by taking $\eta_1$ fraction of nodes with the smallest $w(\cdot)$ values, i.e., nodes lower in the ordering above.
    \item \label{it:cct} Construct $\calG_{\calV'}$ the induced subgraph of $\calV'$ in $\calG$. Let $\calT_{\calV'}$ be a spanning tree of $\calG_{\calV'}$ and $\calE_{\calV'}^c$ be the edges of $\calG_{\calV'}$ outside $\calT_{\calV'}$. 
    \item For a second hyperparameter $\eta_2$, we randomly remove $\eta_2$ fractions of edges in $\calE_{\calV'}^c$ (cf.\ \figref{fig:drop}). 
    \item\label{it:G'} Let $\calG'$ be the resulting graph on $\calV$ with the following edge sets: (i) those outside $\calG_{\calV'}$, (ii) those in $\calT_{\calV'}$ and (iii) remaining edges in $\calE_{\calV'}^c$ after dropping. 
    \item Retrain $\mathfrak{M}$ on $\calG'$. 
\end{enumerate}
\end{algorithm}

\begin{figure}[!htb]
    \centering
\includegraphics[scale=0.43]{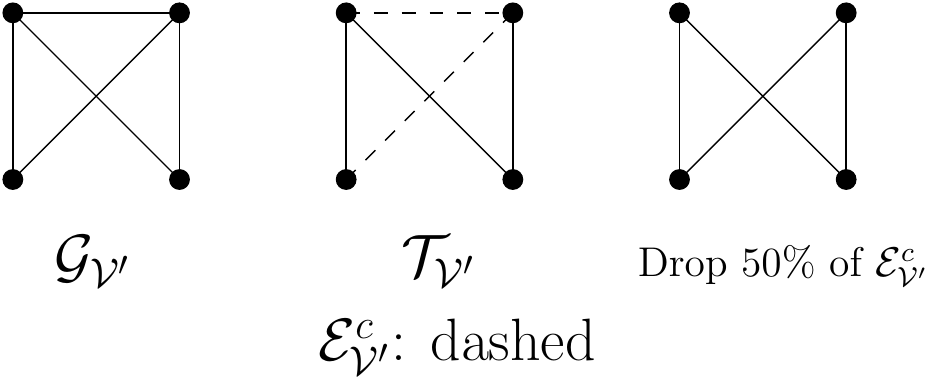}
    \caption{The edge dropping step.}
    \label{fig:drop}
\end{figure}

To give some intuitions, in \cref{algo:pab}, we use $w(\cdot)$ to select nodes with possibly high prediction accuracy. These are then included in the new training set $\frakR'$ with their predicted test labels as the ``ground-truth''. Since the predicted test labels are mostly accurate, the larger training set is expected to lead to better model performance during testing. In \cref{algo:edu}, we use $w(\cdot)$ to identify a set of nodes close in distance to class boundaries. Edge dropping aims to reduce the maximal cut size of this set, which may create a graph bottleneck in view of \cref{sec:gbo}. On the new graph $\calG'$, the base model is expected to learn embeddings that are easier to distinguish among different classes. %(cf.\ \cref{sec:vis}). 
%\red{[Do we have t-sne visualizations to verify this?]}

We remark that each algorithm can be applied as a stand-alone module to any chosen base model, or combined together, which we do in our experiments. The combined approach means that we retrain $\mathfrak{M}$, the base model, with $\calG'$ in step \ref{it:G'} of \cref{algo:edu} as the graph and $\mathfrak{R}'$ in step \ref{it:R'} of \cref{algo:pab} as the training set in the last step. An illustration is shown in \figref{fig:model}. 

For \cref{algo:pab}, the final test accuracy is obtained by tallying the accuracy of all the nodes in $\mathfrak{T}$, i.e., the predicted labels of nodes in $\mathfrak{R}'\backslash \mathfrak{R}$ are from the initial step \ref{it:pab}. In \cref{algo:edu}\ref{it:cct}, we propose not dropping any edge from a spanning tree so that the number of connected components remains the same after edge dropping, which also holds true for the combined approach. %\red{[How about the combined approach?]}

\begin{figure}[!htb]
    \centering
\includegraphics[scale=0.5]{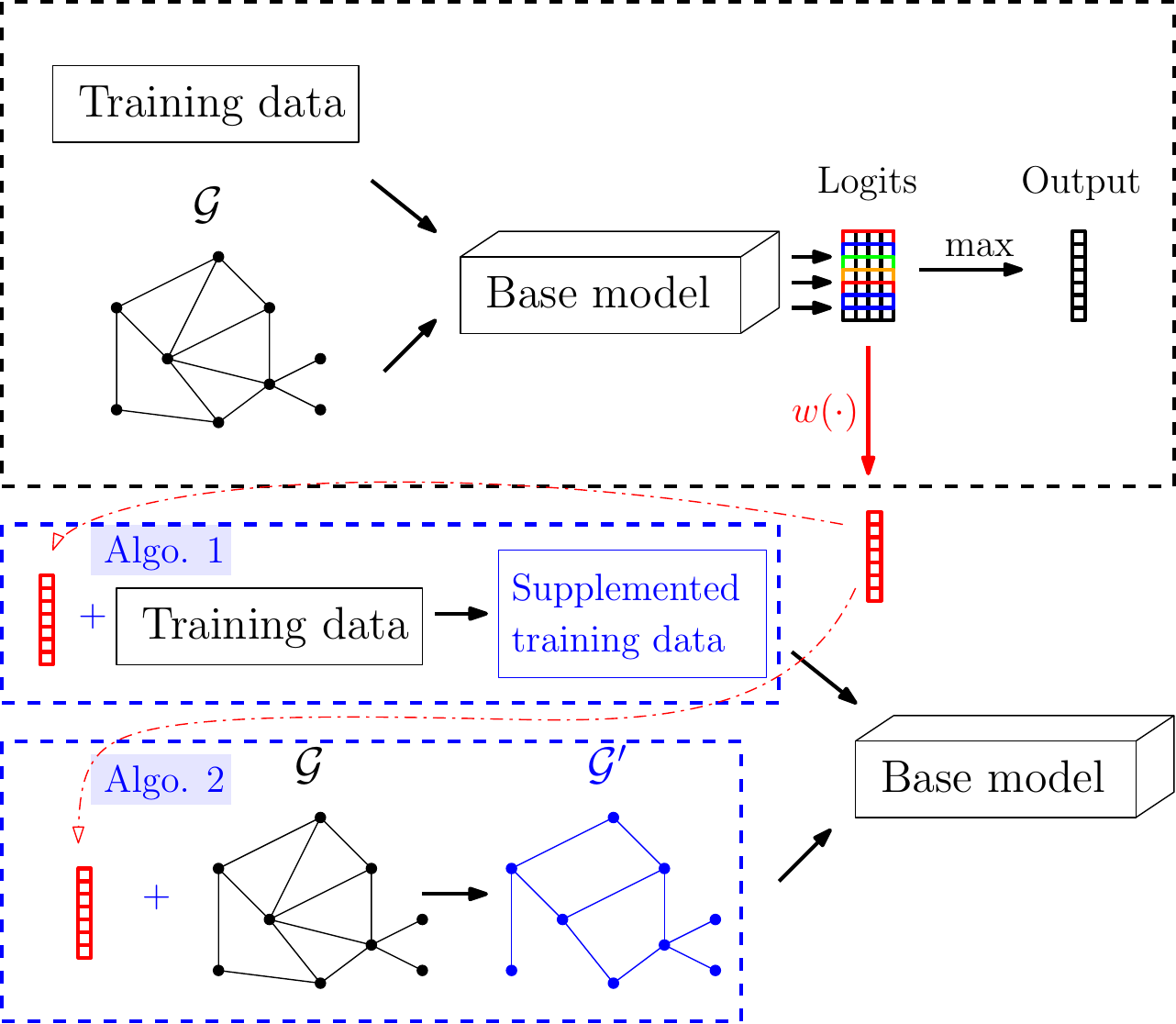}
    \caption{The figure is the scheme of the proposed model. We see that there are two separate modules (in the dashed blue boxes %\red{[Can you label them as Algorithm 1 and 2 in the figure?]}) 
    corresponding to the two algorithms. They can be applied either separately or jointly to the base model.%\red{[Remove arrow from top $\calG$ to bottom $\calG$. Instead, should have arrows from red $w(\cdot)$ at top to the red array inputs at the bottom.]} 
    }
    \label{fig:model}
\end{figure}

\subsubsection*{Related work}
From \cref{sec:gta}, our approach is based on the study of geometric information, more precisely the graph structural information, associated with features and model output. This is not the sole work that emphasizes the importance of geometric knowledge regarding both graph topology and feature embedding. For example, one group of works \citep{Cha19, Gul19, Zhay21} argue that each graph can be associated with a measure of hyperbolicity \citep{Bri99}, and therefore a graph with small hyperbolicity should be studied using hyperbolic geometry \citep{Bac19}. There are also hybrid models combining both hyperbolic and Euclidean geometries \citep{Zhu20}. Some works \citep{Ebl20, Weis21, Lee22} use the concept of simplicial complexes that generalizes graphs to account for higher-order relations among nodes. Our model is different in the sense that our objective is not to find a geometric space best suited for the dataset including both the graph and features. There are also works that tweak the graph topology. For example, DropEdge \citep{Ron20} proposes to randomly drop a fraction of edges in each iteration to alleviate the side-effects of oversmoothing \citet{NT19, Oon20, Che20}. A similar goal is pursued in \citet{Luo21} by filtering out task-specific noisy edges. Our objective, however, is to infer hidden graph structural information associated with label class boundaries. \cref{algo:pab} is similar to self-training \citep{Liq18} by using a part of the test nodes for training (cf.\ \cref{sec:mc}). However, we select test nodes based on the newly introduced $w(\cdot)$, which is different from \citet{Liq18}. Moreover, though \cref{algo:pab} and \cref{algo:edu} are different in nature, they are coherently derived from the same theoretical analysis.    

\section{Experiments}\label{sec:exp}

In this section, we evaluate the proposed model. 

\subsection{Node classification results} \label{sec:cla}

We perform experiments on the node classification problem. The datasets used are Cora, Citeseer, Pubmed, (Amazon) Photo, CS, Airport, and Disease \citep{Sen08, Nam12, Che18, Shc18, Fey19, Cha19}. Our approach requires a base model. We use GCN \citep{Def16}, GAT \citep{Vel18}, DropEdge \citep{Ron20}, GIL \citep{Zhu20}, GraphCON \citep{Rus22}, MaskGAE \citep{Lij22}. %\red{[TO ADD MORE]}\red{[Are we including graph PDE models?] Ans: We will only have time for Graphcon hopefully.}. 
We call our model the \emph{graph neural network using $w(\cdot)$} ($w$GNN), though the name does not explicitly refer to the base model being used. We single out discussions of heterophilic graphs in \cref{sec:het}. More comparisons are given in \cref{sec:mc}.

There are $3$ hyperparameters: $\eta_0$ in \cref{algo:pab} and $\eta_1,\eta_2$ in \cref{algo:edu}. They are tuned using a grid search (of $0.1$ in stepsize) based on the accuracy of the validation set. We propose two ways to apply \cref{algo:edu}. The more principled way is to apply the same base model for \cref{algo:pab} and \cref{algo:edu}. On the other hand, for any dataset, we can also apply \cref{algo:edu} using a fixed model such as GCN once to generate $\calG'$, which is stored for any other models on the same dataset. This is a compromise that is very efficient. We justify this procedure in \cref{sec:mmb}. Other details regarding datasets and source code are in \cref{app:sta}. The results are shown in \cref{tab:ncresult}. In all cases, our model shows a performance improvement. We perform statistical tests and notice that the $p$-value of $37$ among $42$ (about $88\%$) comparisons is $<0.05$, i.e., the improvement of $w$GNN is significant. 

\begin{table*}
\caption{Node classification result. Performance score averaged over ten runs. The best performance is boldfaced. MaskGAE uses a non-standard split for CS, and the results are to show improvements only and are not used for comparison with other benchmarks.}
\label{tab:ncresult}
\centering
\scalebox{0.9}{
\begin{tabular}{@{}lccccccc@{}}
\toprule
Method          & CS                & Photo             & Cora               & Citeseer           & Pubmed            & Airport            & Disease \\\midrule
GCN             & $88.14 \pm 0.42 $ & $90.66 \pm 0.52 $ & $80.65 \pm 0.49$   & $71.23 \pm 0.66$   & $79.03 \pm 0.38$  & $85.08 \pm 2.02$ &  $87.68 \pm 3.67$\\
\hdashline
$w$GNN          & $89.29 \pm 0.14 $ & $92.35 \pm 0.18 $ & $83.12 \pm 0.31$   & ${\bf 73.95} \pm 0.46$   & $80.48 \pm 0.25$  & $87.77 \pm 1.57$ &  $89.02 \pm 4.33$\\
\hline
GAT             & $88.51 \pm 0.73 $ & $90.36 \pm 0.85 $ & $81.91 \pm 0.48$   & $70.21 \pm 0.52$   & $78.91 \pm 0.42$  & $91.23 \pm 3.40$ &  $83.74 \pm 2.31$\\
\hdashline
$w$GNN          & $89.64 \pm 0.38 $ & $91.82 \pm 0.25 $ & $ 84.84 \pm 0.50$   & $72.85 \pm 0.49$   & $79.59 \pm 0.28$  & ${\bf 93.18} \pm 1.22 $ & $86.42 \pm 1.00$\\
\hline
DropEdge        & $88.27 \pm 0.43 $ & $90.49 \pm 0.64 $ & $81.00 \pm 0.55$   & $70.72 \pm 0.56$   & $79.18 \pm 0.33$  & $86.80 \pm 1.50 $ & $87.72 \pm 2.60$\\
\hdashline
$w$GNN          & $88.92 \pm 0.27 $ & $92.09 \pm 0.24 $ & $83.89 \pm 0.36$   & $73.13 \pm 0.22$   & $80.20 \pm 0.21$  & $87.17 \pm 1.59 $ & ${\bf 92.32} \pm 0.43$ \\
\hline
MaskGAE*        & $92.72 \pm 0.11 $ & $91.55 \pm 0.23 $ & $82.03 \pm 0.76$   & $70.10 \pm 1.37$   & $80.11 \pm 0.51$  & $70.56 \pm 1.25$ & $70.51 \pm 3.57$ \\
\hdashline
$w$GNN          & $95.16 \pm 0.68$ & ${\bf 93.50} \pm 0.28 $ & $82.85 \pm 0.26$   & $72.38 \pm 0.97$   & ${\bf 81.78} \pm 0.22$  &$85.62 \pm 1.00$ & $71.77\pm 4.26$ \\
\hline
GIL             & $88.69 \pm 0.93 $ & $89.60 \pm 1.30 $ & $79.65 \pm 1.38$   & $66.43 \pm 1.56$   & $77.18 \pm 1.00$  & $90.34 \pm 1.29$ & $89.96 \pm 1.02$ \\
\hdashline
$w$GNN          & $91.44 \pm 0.18 $ & $91.88 \pm 0.32 $ & $83.16 \pm 0.53$   & $69.51 \pm 0.45$   & $80.73 \pm 0.34$  &$91.15 \pm 1.05$ & $92.01 \pm 0.27$ \\
\hline
GraphCON        & $90.19 \pm 0.70$ & $90.04\pm 0.54$ & $82.36\pm 0.84$  & $70.80\pm 1.40$  & $79.11\pm 1.78$  & $57.32\pm 1.87$ & $70.32\pm 4.45$ \\
\hdashline
$w$GNN          & ${\bf 93.03} \pm 0.64$ & $93.08\pm 0.37$ & ${\bf 85.23}\pm 0.69$  & $71.40\pm 0.87$  & $79.70\pm 1.07$  & $63.82\pm 1.81$ & $71.73\pm 3.18$ \\
\bottomrule
\end{tabular}}%}
\end{table*}

\subsection{Heterophilic graphs} \label{sec:het}

Recall that for a dataset, the graph is heterophilic if there are many edges connecting nodes with different label classes. In this subsection, we study the performance of $w$GNN on datasets Texas and Chameleon \citep{Zhu21} with heterophilic graphs. In addition to the base models in \cref{sec:cla}, we also consider ACM-GCN \citep{Lua22}, which is dedicated to addressing graph heterophily. We first show the results in \cref{tab:het}. Similar to \cref{sec:cla}, each $w$GNN is paired with its base model. We see that $w$GNN is able to improve all the base models, including ACM-GCN.

\begin{table*}[h]
\caption{Results for datasets with heterophilic graphs} \label{tab:het}
\centering
\scalebox{0.82}{
\begin{tabular}{ c : c c : c c : c c : c c : c c : c c} 
 \toprule
  &  GCN &  $w$GNN &  GAT &  $w$GNN &  Dropedge &  $w$GNN &  GIL &  $w$GNN &  GraphCON &  $w$GNN & ACM-GCN & $w$GNN\\ 
  \midrule
  Texas & $54.60\atop\pm 6.47$ & $58.11\atop\pm 9.72$ & $52.70\atop\pm 8.38$ & $55.94\atop\pm 9.71$ & $55.95\atop\pm 6.25$ & $58.11\atop\pm 9.97$ & $57.84\atop\pm 5.44$ & $61.89\atop\pm 6.91$ & $80.81\atop\pm 4.12$ & $85.95\atop\pm 4.73$ & $88.38\atop\pm 3.21$ & ${\bf 90.27}\atop\pm 4.05$\\ 
  \midrule
  Chameleon & $63.62\atop\pm 3.14$ & $64.45\atop\pm 2.40$ & $64.25\atop\pm 3.41$ & $66.12\atop\pm 2.58$ & $63.57\atop\pm 1.97$ & $64.61\atop\pm 2.39$ & $64.41\atop\pm 1.96$ & $64.89\atop\pm 1.04$ & $50.77\atop\pm 2.35$ & $55.53\atop\pm 2.26$ & $65.92\atop\pm 2.32$ & ${\bf 80.79} \atop\pm 1.95$ \\ 
 \bottomrule
\end{tabular}}
\end{table*}

We offer possible explanations for why $w$GNN also works for datasets with heterophilic graphs. In the extreme case that every edge connects a pair of nodes of different classes, then we have a $k$-partite graph, where $k$ is the number of label classes. Therefore, a heterophilic graph is almost $k$-partite with very sparse connections within each of the $k$-components. In principle, our approach can also be helpful for heterophilic graphs. For example, \cref{algo:edu} reduces the connections among different components, and message passings in GNNs rely more on connections within each component. This is favorable for a shallow GNN model, where each node aggregates information only from close neighbors in message passing. Therefore, the predictive power of the model depends largely on the small neighborhood of each node. 

The challenge for a heterophilic graph is that each node is likely to receive ``noisy information'' from neighbors belonging to different classes, due to a large number of edges connecting different types of nodes. For $w$GNN, though edge dropping cannot add connections between nodes of the same class, it can reduce connections between nodes of different classes with a high chance (due to heterophily). When this happens, during message passing, each node can potentially receive less ``noisy information'' from nodes of different classes. Consequently, the contribution from nodes of the same class increases. As illustrated in \figref{fig:edgedrop}, $v$ initially aggregates information from one node of class 1 and two nodes from class 2. After edge dropping in $w$GNN, $v$ has only one node from class 2 as a neighbor. Heuristically, contributions (in fraction) from the same class 1 have increased by $50\%$.  

\begin{figure}
    \centering
    \includegraphics[scale=0.28]{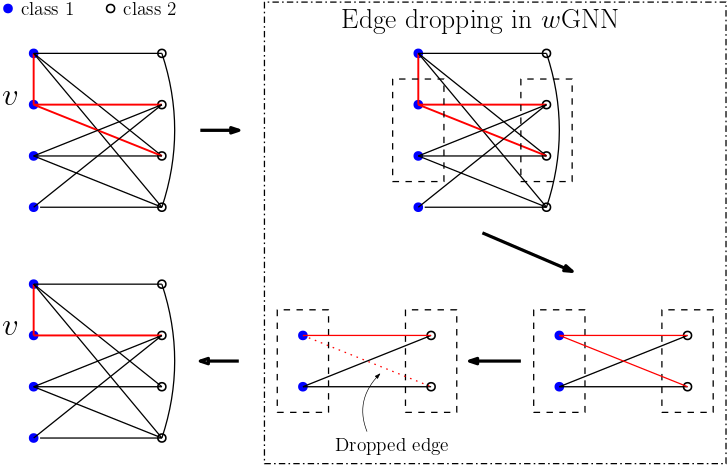}
    \caption{An illustration of the edge dropping step \cref{algo:edu}.} \label{fig:edgedrop}
\end{figure}

In addition, if a base model handles heterophilic edges well, then $w$GNN (in particular \cref{algo:pab}) allows us to have more reliable nodes for each class. In message passing, each node again receives more reliable information. This is possibly the reason why $w$GNN can also improve the specialized model ACM-GCN.

\subsection{Ablation study} \label{sec:as}

Our approach has two submodules \cref{algo:pab} and \cref{algo:edu}, each of which can be applied as a stand-alone algorithm to the base model. In this subsection, we perform the ablation study to demonstrate that both algorithms contribute to the observed performance. For clarity, we use $w$GNN1 (resp.\ $w$GNN2) for the variant that applies \cref{algo:pab} (resp.\ \cref{algo:edu}) only. The Disease dataset is excluded from the study because the graph has a tree structure and \cref{algo:edu} is not involved. The results are shown in \cref{tab:ablation}. Apart from seeing contributions from both algorithms, \cref{algo:pab} appears to have a stronger impact.  

\begin{table*}
\caption{The ablation study: performance score averaged over ten runs. The best performance is boldfaced.}
\label{tab:ablation}
\centering
\scalebox{0.9}{
\begin{tabular}{@{}lcccccc@{}}
\toprule
Method          & CS & Photo & Cora               & Citeseer           & Pubmed             & Airport             \\\midrule
$w$GNN (GCN)    &  $ 89.29 \pm 0.14$ & ${\bf 92.35} \pm 0.18 $ & $ 83.12 \pm 0.31$   & ${\bf 73.95} \pm 0.46$   & ${\bf 80.48} \pm 0.25$   & $ 87.77 \pm 1.57$   \\
\hdashline
$w$GNN1         & $ 89.29 \pm 0.14 $ & $91.93 \pm 0.25 $ & $82.88 \pm 0.41$   & $73.55 \pm 0.46$   & $80.36 \pm 0.39$   & $87.55 \pm 1.89$   \\
\hdashline
$w$GNN2         & $ 88.54 \pm 0.37 $ & $91.23 \pm 0.30 $ & $81.10 \pm 0.25$   & $71.98 \pm 0.52$   & $79.19 \pm 0.37$   & $86.07 \pm 1.52$  \\
\hline
$w$GNN (GAT)    & ${\bf 89.64} \pm 0.38 $ & $91.82 \pm 0.25 $ & ${\bf 84.84} \pm 0.50$   & $72.85 \pm 0.49$   & $79.59 \pm 0.28$   & ${\bf 93.18} \pm 1.22$   \\
\hdashline
$w$GNN1         & $89.54 \pm 0.33 $ &  $91.50 \pm 0.35 $ & $83.99 \pm 0.29$   & $72.35 \pm 0.50$   & $79.59 \pm 0.28$  & $92.11 \pm 0.95$   \\
\hdashline
$w$GNN2         & $88.60 \pm 0.66 $ & $89.97 \pm 0.86 $ & $82.03 \pm 0.63$   & $70.55 \pm 0.62$   & $78.97 \pm 0.32$   &  $91.80 \pm 1.07$  \\
\bottomrule
\end{tabular}}
%}
\end{table*}

\subsection{Further analysis}

In this subsection, we present further analysis of $w$GNN. 

\subsubsection{Choice of parameters} \label{sec:cop}

There are $3$ hyperparameters $\eta_0, \eta_1, \eta_2$ in the model. In this subsection, we study how they impact the model performance. It would be cumbersome to present the results for all possible combinations of $\eta_0,\eta_1,\eta_2$. Instead, we choose a few typical combinations to show the overall pattern. We choose $\eta_1=\eta_2\in [0.1:0.1:0.8]$ so that both light edge drop and heavy edge drop are considered. We let $\eta_0 \in \{0.2,0.4,0.6,0.8\}$. GCN is used as the base model for the study, and heatmaps for the results (Cora and Citeseer datasets) are shown in \figref{fig:para}. We see that the model performance is generally better for larger $\eta_0$, while the performance can drop if $\eta_0$ is large enough. Intuitively, if $\eta_0$ is too large, then \cref{algo:pab} may introduce more errors that may offset any benefits it brings about. On the other hand, we observe that good choices of $\eta_1,\eta_2$ that work for all $\eta_0$ happen at $\approx 0.6, \approx 0.5$ for Cora and Citeseer, respectively. This demonstrates the useful role played by \cref{algo:edu}.
\begin{figure}[h]
    \centering
    \includegraphics[scale=0.3, trim={1.5cm 6.5cm 0.5cm 6.5cm},clip]{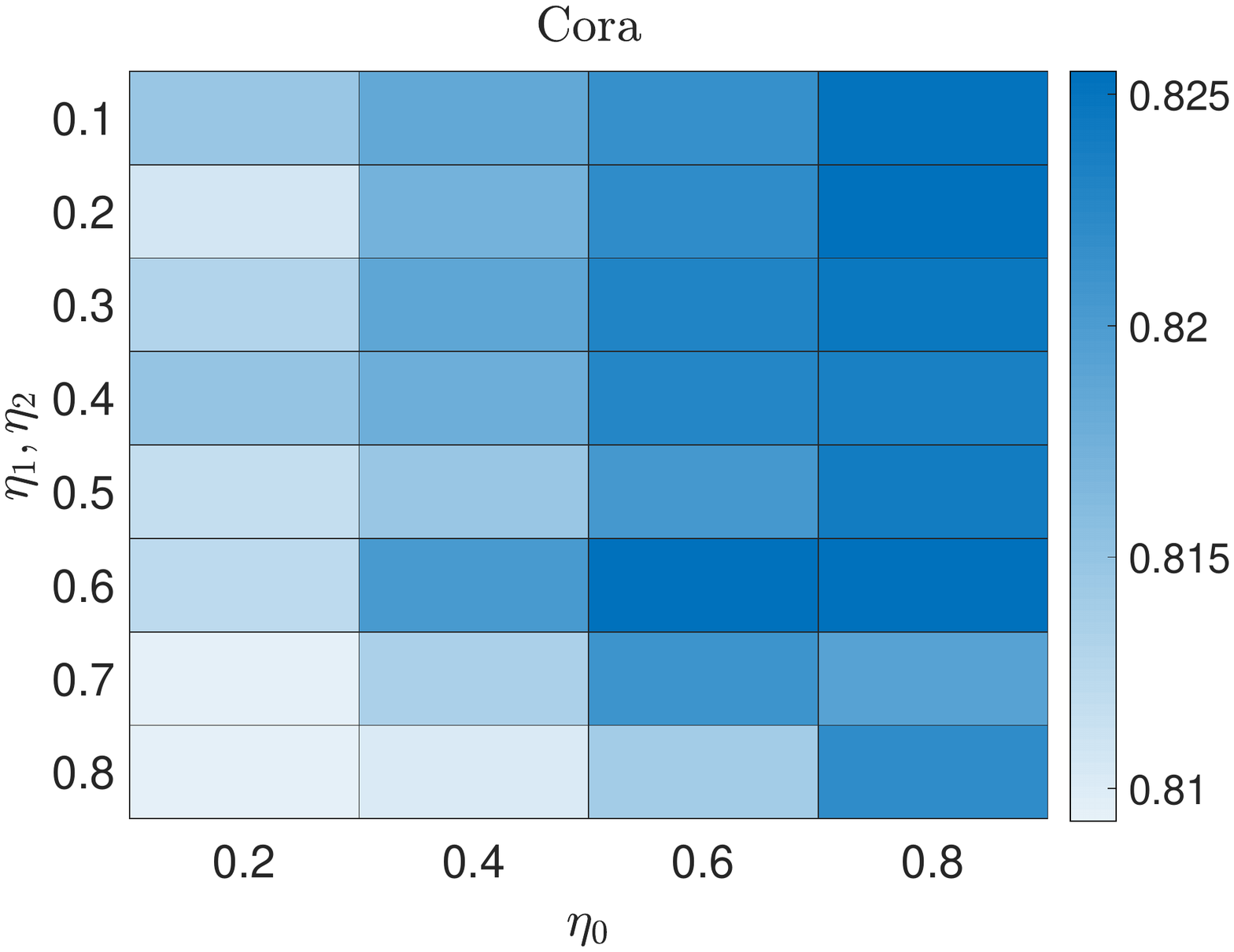}
    \includegraphics[scale=0.3, trim={1.5cm 6.5cm 0.5cm 6.5cm},clip]{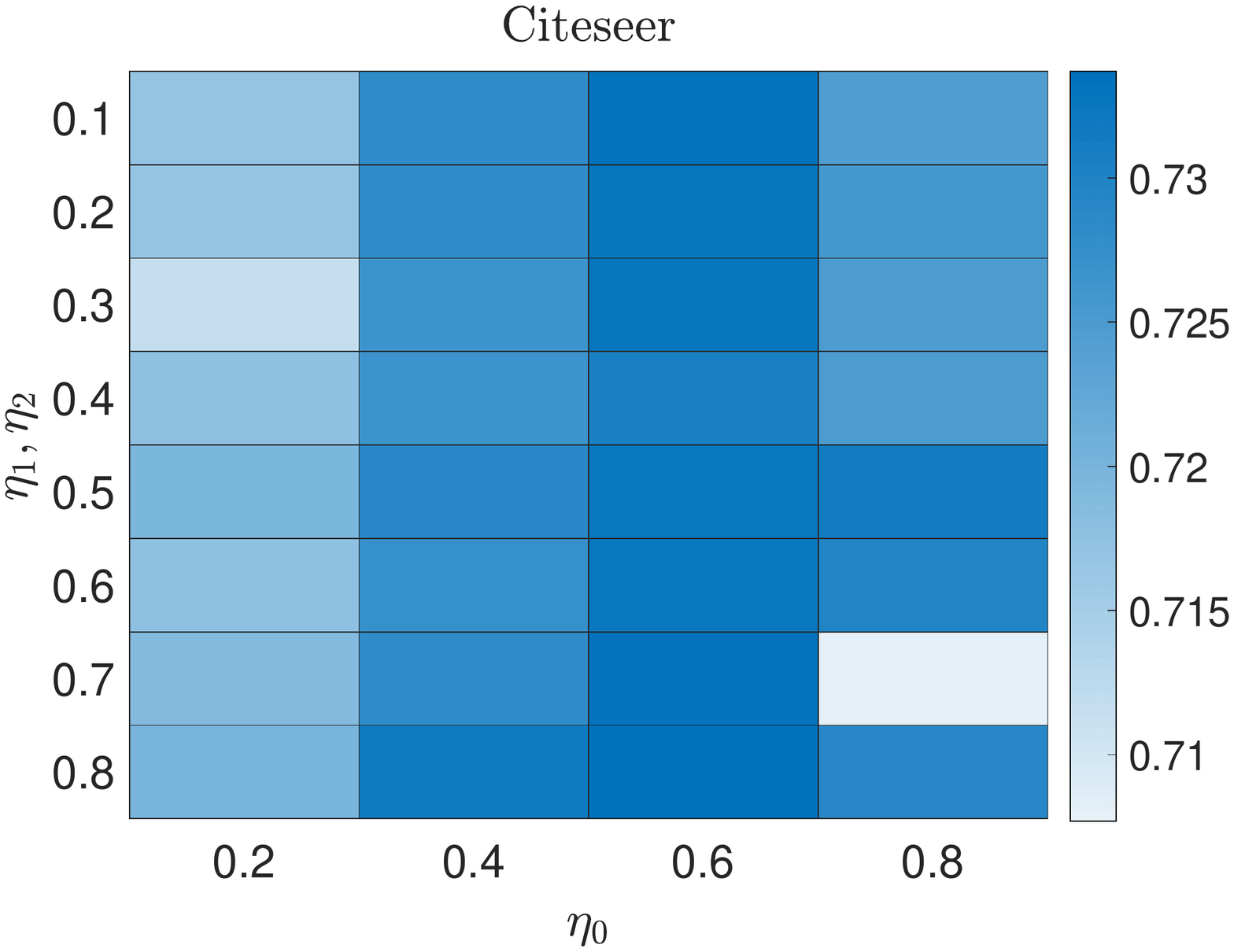}
    \caption{Performance against parameter choices.}
    \label{fig:para}
\end{figure}

\subsubsection{Model mismatch between \cref{algo:pab} and \cref{algo:edu}} \label{sec:mmb}

As discussed in \cref{sec:cla}, an experimental option is to use the same $\calG'$ generated by \cref{algo:edu} from a single model such as GCN for any algorithm on the same dataset. Here, we analyze this by studying model mismatch between \cref{algo:pab} and \cref{algo:edu} using the following experiment.

We use GCN as the base model for Cora and Citeseer datasets. To apply \cref{algo:edu} to generate $\calG'$, we compare using both GCN and GAT, while the latter accounts for the model mismatch. The hyperparameter $\eta_0$ is chosen to be $0.8$ for Cora and $0.6$ for Citeseer for better performance, as we have observed in \figref{fig:para}. The parameters $\eta_1 = \eta_2$ vary from $0.1$ to $0.8$ as in \cref{sec:cop}. 

\begin{figure}[!htb]
    \centering
    \includegraphics[scale=0.33, trim={0 0cm 0 0cm},clip]{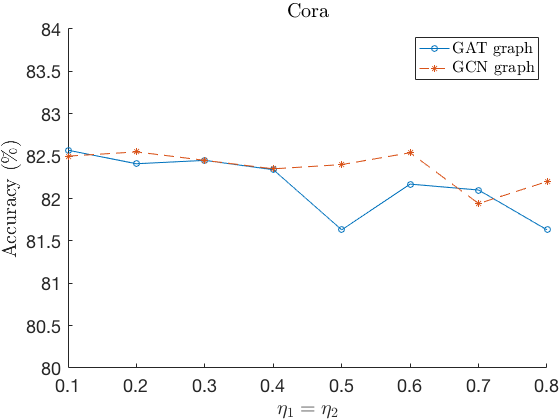}
    \includegraphics[scale=0.33, trim={0 0cm 0 0cm},clip]{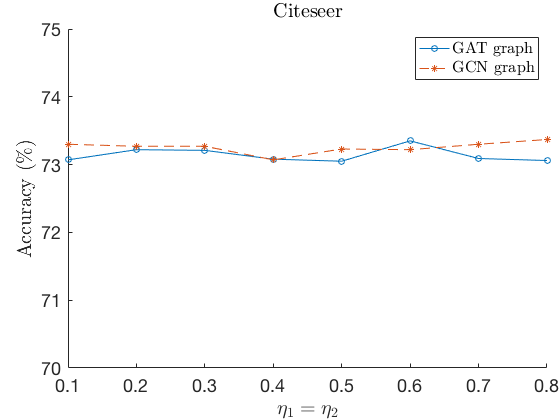}
    \caption{Performance for model mismatch between \cref{algo:pab} and \cref{algo:edu}.}
    \label{fig:gatgraph}
\end{figure}

The results are shown in \figref{fig:gatgraph}. We see that for $\eta_1=\eta_2 \leq 0.4$, when the accuracies are relatively high, the mismatch does not have a large impact on the outcome. For large $\eta_1=\eta_2$, there is an expected difference as heavier edge drop may cause a larger difference in graph topologies. However, for both experiments, the optimal accuracy does not suffer much from using a graph with a mismatched generating base model. Therefore, in practice, we may run \cref{algo:edu} and store the generated graph as an overhead, in resource-constrained applications.   

\subsubsection{Bad base models}
Our approach relies on a base model. In practice, it is possible that a base model has poor performance, particularly at the initial investigation stage of a new dataset. In this subsection, we study the performance of our approach if such a base model is given. The candidates for base models are plain GAT with more layers. It is observed that the performance can be poor if more layers are added without fundamentally tweaking the GAT model structure, due to reasons such as oversmoothing and oversquashing \citep{Top22}. We consider variants of GAT with $L = 2$ to $8$ layers, with deteriorating performance. We choose $\eta_0=0.4, 0.8$ and run the experiments fixing $\eta_1=\eta_2=0$. The results for the Cora dataset are shown in \figref{fig:badmodel}. Unless $L= 8$ when the base model accuracy is very low, our approach can reasonably improve the base model performance. For example, when $L=7$, choosing $\eta_0=0.4$ improves the accuracy of the based model by $\approx 15\%$. Moreover, in this case, $\eta_0=0.4$ is much better than $\eta_0=0.8$, in contrast to $L=2$ when $\eta_0=0.8$ has higher accuracy. This is because when the base model is highly inaccurate, then we are likely supplementing ``datapoints with bad quality'' in \cref{algo:pab} if $\eta_0$ is set to be large. %\red{[How about for $w$GNN2?]}. Ans: I do not have time to do the study for $w$GNN2. For $w$GNN, I did experiment for some choices of $\eta_1=\eta_2$ and $\eta_0=0.4,0.8$, the results are similar but slightly worse in overall accuracy.
This might also be the reason that when $L=8$, applying our approach alone is insufficient to enhance the base model performance to a reasonable level. The findings also suggest $w(\cdot)$ can help with efficiently selecting useful nodes as long as the base model has reasonable performance.   

\begin{figure}[!htb]
    \centering
    \includegraphics[scale=0.33, trim={0 0cm 0 0cm},clip]{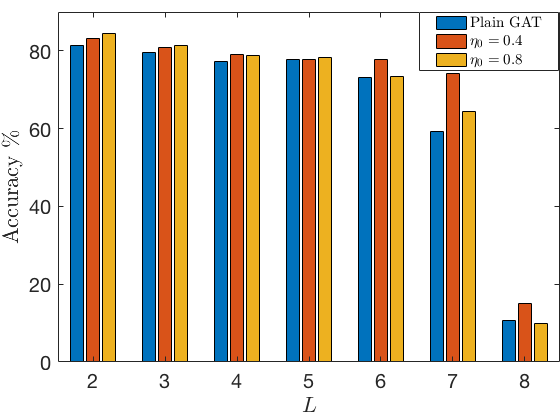}
    \caption{Performance based on GAT with $L=2$ to $8$ layers.} 
    \label{fig:badmodel}
\end{figure}

\section{Conclusion}
In this paper, we study the logits from the intermediate step of a typical GNN model, using a non-uniformity function. We gain hidden graph structural insights and propose a GNN model based on theoretical findings. Our model requires a base GNN model and we demonstrate with experiments to show that our approach can significantly improve the performance of base models in most cases.

\section*{Acknowledgements}
This work is supported by the Singapore Ministry of Education Academic Research Fund Tier 2 grant MOE-T2EP20220-0002, and the National Research Foundation, Singapore and Infocomm Media Development Authority under its Future Communications Research and Development Programme. 

S.\ H.\ Lee is also supported by Shopee Pte.\ Ltd.\ under the Economic Development Board Industrial Postgraduate Programme (EDB IPP). H. Meng and J. Yang are supported in part by the Young Scientists Fund of the National Natural Science Foundation of China (No.62106082). 

The computational work for this article was partially performed on resources of the National Supercomputing Centre, Singapore (\url{https://www.nscc.sg}).
\newpage

%\red{[Is it the conference format practice to use only conference name abbreviations like ``ICML'' and ``AAAI'' without writing out the full names? NIPS should be NeurIPS or in full ``Advances in Neural Information Processing Systems''.]}
\bibliography{ref}
\bibliographystyle{icml2023}

%%%%%%%%%%%%%%%%%%%%%%%%%%%%%%%%%%%%%%%%%%%%%%%%%%%%%%%%%%%%%%%%%%%%%%%%%%%%%%%
%%%%%%%%%%%%%%%%%%%%%%%%%%%%%%%%%%%%%%%%%%%%%%%%%%%%%%%%%%%%%%%%%%%%%%%%%%%%%%%
% APPENDIX
%%%%%%%%%%%%%%%%%%%%%%%%%%%%%%%%%%%%%%%%%%%%%%%%%%%%%%%%%%%%%%%%%%%%%%%%%%%%%%%
%%%%%%%%%%%%%%%%%%%%%%%%%%%%%%%%%%%%%%%%%%%%%%%%%%%%%%%%%%%%%%%%%%%%%%%%%%%%%%%
\newpage
\appendix
\onecolumn
\section{List of notations}

For easy reference, we list the most used notations in \cref{tab:lon}.

\begin{table}[!ht]
\caption{List of notations}
\label{tab:lon}
\begin{center}
\scalebox{1}{\begin{tabular}{ c c } 
  \toprule
  Graph, Vertex set, Edge set & $\calG, \calV, \calE$ \\
 \midrule
  Nodes & $v,v'$ \\
 \midrule 
  Training, Test sets & $\mathfrak{R}, \mathfrak{T}$\\
  \midrule 
  Class labels & $s,s'\in \mathbb{S}$ \\
  \midrule
  Distribution at $v$ & $\bmu_v$  \\
  \midrule
  Label non-uniformity & $w(\cdot)$ \\
  \midrule
  Graph signals & $\boldf, \boldf', \tilde{\boldf}$ \\
  \midrule
  Graph Laplacian & $L_{\calG}$ \\
  \midrule
  Subsets of nodes & $\calP, \calO, \calV', \calU$ \\
  \midrule
  Level component of $v$ & $\calC_v$ \\
  \midrule
  Boundary & $\partial$ \\
  \midrule
  Induced subgraph & $\calG_{\calV'}$ \\
  \midrule
  Base GNN model & $\mathfrak{M}$\\
 \bottomrule
\end{tabular}} 
\end{center}
\end{table}

\section{Theoretical results and discussions} \label{sec:pro}

In this appendix, we prove all the results of the paper and present other theoretical findings. We first define the Wasserstein distance \citep{Vil09} and justify (\ref{eq:wsi}).

\begin{Definition} \label{defn:fmm}
For two probability distribution $\mu_1, \mu_2$ with finite second momments on a metric space $\mathbb{M}$, the \emph{$2$-Wasserstein metric} $W(\mu_1,\mu_2)$ between $\mu_1,\mu_2$ is defined by
\begin{align*}
    W(\mu_1,\mu_2)^2 = \inf_{\gamma\in \Gamma(\mu_1,\mu_2)}\int d(x,y)^2 \ud\gamma(x,y),
\end{align*}
where $\Gamma(\mu_1,\mu_2)$ is the set of \emph{couplings} of $\mu_1,\mu_2$, i.e., the collection of probability measures on $\mathbb{M} \times \mathbb{M}$ whose marginals are $\mu_1$ and $\mu_2$, respectively.
\end{Definition}

Suppose $\mathbb{S} = \{s_1,\ldots,s_m\}$ is a finite discrete set and $d$ is the discrete metric on $\mathbb{S}$. For distributions $\mu,\nu$ on $\mathbb{S}$, let $(\mu(s_i))_{1\leq i\leq n}$ and $(\nu(s_i))_{1\leq i\leq n}$ be their respective probability weights. 

\begin{Lemma} \label{lem:wmn}
\begin{align*}
W(\mu,\nu)^2 = \frac{1}{2}\sum_{1\leq i\leq m} |\mu(s_i)-\nu(s_i)|.
\end{align*}
\end{Lemma}

\begin{proof}
This is a known result and we present an elementary self-contained proof here. Let $\gamma = \big(\gamma(s_i,s_j)\big)_{1\leq i,j\leq m}$ be in $\Gamma(\mu,\nu)$. We have
\begin{align} \label{eq:sum}
\begin{split}
    &\sum_{1\leq i\leq m}\sum_{1\leq j\leq m} \gamma(s_i,s_j)d(s_i,s_j)^2 \\
    &= \sum_{1\leq i\leq m}\sum_{1\leq j\neq i\leq m} \gamma(s_i,s_j) \\
    &= \sum_{1\leq i\leq m} \parens*{\sum_{1\leq j\leq m} \gamma(s_i,s_j) - \gamma(s_i,s_i)}  \\
    &= \sum_{1\leq i\leq m} \parens*{\mu(s_i) - \gamma(s_i,s_i)} \\
    &\geq \sum_{1\leq i\leq m} \parens*{\mu(s_i) - \min(\mu(s_i),\nu(s_i)}.
    \end{split}
\end{align}
As $W(\mu,\nu)^2$ is defined by taking the infimum of the left-hand side over all $\gamma \in \Gamma(\mu,\nu)$, we have $W(\mu,\nu)^2\geq \sum_{1\leq i\leq m} \big(\mu(s_i) - \min(\mu(s_i),\nu(s_i)\big)$. By the same argument, we also have $W(\mu,\nu)^2\geq \sum_{1\leq i\leq m} \big(\nu(s_i) - \min(\mu(s_i),\nu(s_i)\big)$. Summing up these two inequalities, we have 
\begin{align*}
2W(\mu,\nu)^2 & \geq \sum_{1\leq i\leq m} \big(\mu(s_i)+\nu(s_i) - 2\min(\mu(s_i),\nu(s_i)\big) \\
& = \sum_{1\leq i\leq m} |\mu(s_i)-\nu(s_i)|.
\end{align*}
Therefore, to prove the lemma, it suffices to show that there is a $\gamma$ such that $\gamma(s_i,s_i) = \min\big(\mu(s_i),\nu(s_i)\big)$. For this, we prove a slightly more general claim: if non-negative numbers $(x_i)_{1\leq i\leq m}$ and $(y_i)_{1\leq i\leq m}$ satisfy $\sum_{1\leq i\leq m}x_i = \sum_{1\leq j\leq m}y_i = a$, then there are non-negative $(z_{i,j})_{1\leq i,j\leq m}$ such that $\sum_{1\leq j\leq m}z_{i,j}=x_i, 1\leq i\leq m$, $\sum_{1\leq i\leq m}z_{i,j}=y_j, 1\leq j\leq m$, and $z_{i,i}=\min(x_i,y_i), 1\leq i\leq m$.

We prove this by induction on $m$. The case for $m=1$ is trivially true by taking $z_{1,1} = x_1 = y_1$. For $m\geq 2$, without loss of generality, we assume that $x_1\geq y_1$ and $x_2\leq y_2$. Then we choose $z_{1,1}=y_1$, $z_{2,2} = x_2$, $z_{1,j}=0, 1< j\leq m$ and $z_{i, 2} = 0, 1\leq i
\neq 2\leq m$. As a result, we form another two sequences of non-negative numbers $x_1-y_1,x_3,\ldots,x_m$ and $y_2-x_2,y_3,\ldots,y_m$ with both summing to $a - x_2 - y_1$. By the induction hypothesis, we are able to find non-negative $(z'_{i,j})_{1\leq i,j\leq m-1}$ for the two new sequences of length $m-1$ each. It suffices to let $z_{i,j}= z'_{i-1,j-1}$ for $i>1$ or $j>2$ and $z_{i,1}=z'_{i-1,1}$ for $i>1$ (illustrated in \cref{fig:cvariable}). This proves the claim and hence the lemma.
\begin{figure}[!htb]
\centering
\includegraphics[width=0.4\columnwidth, trim=0cm 0cm 0cm 0cm,clip]{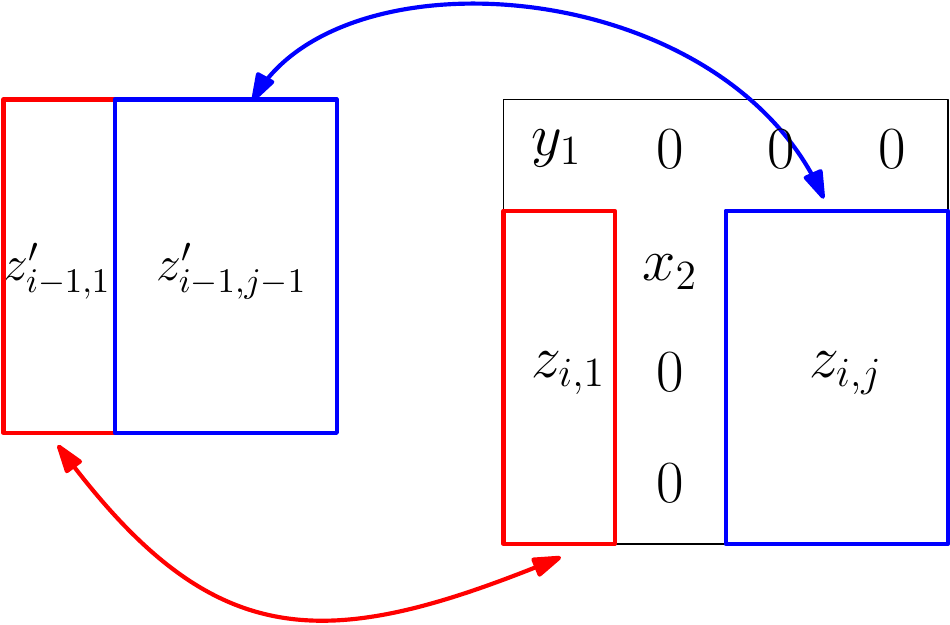}
\caption{The relations between $z_{i,j}$ and $z'_{i,j}$.} \label{fig:cvariable}
\end{figure} 
\end{proof}

Regarding non-uniformity defined by (\ref{eq:wsi}), it suffices to let $\mu$ be $\bmu_v$ and $\nu$ be the uniform distribution on $\mathbb{S}$. Therefore. $w(v)$ is nothing but $2W(\bmu_v,\nu)^2$ for the uniform distribution $\nu$ on $\mathbb{S}$. 

We now move on to the results in \cref{sec:gta} by first proving the averaging property \cref{lem:fev}.

\begin{proof}[Proof of \cref{lem:fev}]
Taking partial derivatives of ${\boldf'}^\top L_{\calG}{\boldf'}$ w.r.t.\ $\boldf'_v$ and setting it to be $0$ yield the stated identity.
\end{proof}

Next, we need the following result to justify using $\boldf$ and $(\bmu_v(s))_{v\in \calV}$ have similar smooth interpolations, in view of the fact that $\boldf$ is an approximation of $(\bmu_v(s))_{v\in \calV}$. 

\begin{Lemma}
If $\calG$ is connected and $\calO$ is non-empty, then smooth interpolation $\boldf \mapsto \widetilde{\boldf}$ as a function $\mathbb{R}^n \to \mathbb{R}^n$ is well-defined and linear.
\end{Lemma}

\begin{proof}
To see this, setting the partial derivatives of ${\boldf'}^\top L_{\calG}{\boldf'}$ to be $0$ yields a linear relation between $\bg$ and $\bh$, where $\bg$ (resp.\ $\bh$) is the subvector of $\widetilde{\boldf}$ corresponds to nodes $\calV\backslash \calO$ (resp.\ $\calO$). The relation takes the form $L_1\bg = L_2\bh$, where $L_1$ is the submatrix of $L_{\calG}$ whose rows and columns are indexed by $\calV\backslash \calO$. It suffices to show that $L_1$ is invertible. As $L_1$ has a strictly smaller size, by the eigenvalue interlacing property \citep{Hwa04}, we only need to show the case for $\calO=\{v\}$ being a singleton. Let $N_v$ be the neighbors of $v\in \calG$. Notice that $L_1 = D + L'$, where $L'$ is the Laplacian of the induced graph $\calG'$ on $\calV\backslash \{v\}$ and $D$ is the diagonal matrix with $1$ at the entries indexed by $N_v$. If $L_1\bx =0$, then $\bx_u = 0$ for $u\in N_v$ and so does $\bx_{v'}$ for any $v'$ in the connected component of $\calG'$ containing $u$. As $\calG$ is connected, this forces $\bx$ to be the zero vector and the result is proved.
\end{proof}

Recall that $\boldf_v$ is observed either $0$ or $1$ at $v\in \calO$. An immediate consequence of the averaging property is the following.

\begin{Corollary}
    For every $v\in \calV$, $0\leq \widetilde{\boldf}_v\leq 1$.
\end{Corollary}

\begin{proof}
Suppose on the contrary that there is a $v$ such that $\widetilde{\boldf}_v > 1$. Without loss of generality, we assume that $\widetilde{\boldf}_v = \max\{\widetilde{\boldf}_{v'},v'\in \calV\}$. By the averaging property, the signal values of all the neighbors of $v$  are the same as $\widetilde{\boldf}_v$. As $\calG$ is connected, the same holds for the value of $\widetilde{\boldf}$ at every node of $\calG$, which is a contradiction. The same argument shows that $\widetilde{\boldf}_v \geq 0$ for $v\in \calV$.
\end{proof}

Suppose $\calV'$ is a subset of $\calV$. We define the \emph{contraction} $\calG_{/\calV'}$ of $\calG$ w.r.t.\ $\calV'$ as follows. The vertex set of $\calG_{/\calV'}$ is $(\calV\backslash \calV') \cup \{u\}$. For $(v,v') \in \calE$ and $v,v' \in \calV\backslash \calV'$, the edge $(v,v')$ remains in $\calG_{/\calV'}$. For $(v,v')\in \calE$ and $v\notin \calV', v'\in \calV'$, there is the edge $(v,u)$ in $\calG_{/\calV'}$. Intuitively, we have replaced the entire vertex set $V'$ be a single node $u$, and thus called a contraction.

\begin{Lemma} \label{lem:gic}
    $\calG_{/\calV'}$ is connected.
\end{Lemma}

\begin{proof}
    Consider any two distinct nodes $v,v' \in \calG_{/\calV'}$. If $v\neq u, v'\neq u$, then they are connected by a path $\calP$ in $\calG$. In $\calG_{/\calV'}$, they are connected by the path $\calP_{/(\calP\cap V')}$. 

    Otherwise, without loss of generality, we assume $v=u$ and $v'\neq u$. Let $w(\cdot)$ be any node in $\calV'$ and $\calP$ be a path connecting $w(\cdot)$ and $v'$ in $\calG$. Then $\calP_{/(\calP\cap \calV')}$ is a path connecting $v=u$ and $v'$ in $\calG_{/\calV'}$. 
\end{proof}

We can now prove \cref{prop:fev}. 

\begin{proof}
    Let $\calG'$ be the graph obtained from $\calG$ be contracting all the level-components of $\calG$, one at a time. The graph $\calG'$ is connected by \cref{lem:gic}. Denote the vertices of $\calG'$ by $c_v$ where $\calC_v$ is a level-component of $\calG$. It is possible that $c_v = c_{v'}$ as long as $v'\in \calC_v$. In addition, $\calG$ is no longer simple as there can be multiple edges between the same pair of nodes. Each edge in $\calG'$ is associated with a unique edge $(v,v') \in \calE$. Therefore, for convenience, we remain using $(v,v')$ for the edge in $\calG'$.
    
    The signal $\widetilde{\boldf}$ induces a signal $\widetilde{\boldf}'$ on $\calG'$ by $\widetilde{\boldf}'_{c_v} = \widetilde{\boldf}_v$, this is well-defined as $\widetilde{\boldf}$ has the same value over a level-component. Moreover, if $c_v$ and $c_{v'}$ are connected by an edge in $\calG'$, then $\widetilde{\boldf}'_{c_v} \neq \widetilde{\boldf}_v$. We give $\calG'$ an orientation by requiring $(v,v')$ is oriented from $v$ to $v'$ if $\widetilde{\boldf}_v < \widetilde{\boldf}_{v'}$, or equivalently $\widetilde{\boldf'}_{c_v} < \widetilde{\boldf'}_{c_{v'}}$. By the averaging property, if $(v,v')$ is an oriented edge, then there is an oriented edge $(v',u')$ unless $v'\in \calO_1$. Similarly, there is also an oriented edge $(u,v)$ unless $v\in \calO_0$. Pairs of directed edges $(u,v)$ and $(v,w)$ are called \emph{matched}.
    
    For $v\in \calV$, we consider $c_v$ of $\calG'$. We can consecutively find matching edges in both directions until reaching $c_{v_0}, v_0\in \calO_0$ and $c_{v_m}, v_m\in \calO_1$ respectively. Connected the two paths at $c_v$, we obtain $\calP$ passing through $c_v$ such that consecutive edges are matched. By construction, $\calP$ also gives a path (with the same edge labels) that satisfies (a)-(c) for $v$.
\end{proof}

\cref{prop:fev} is used to deduce \cref{coro:fau} as follows.

\begin{proof}
    We assume that $\calG_{\calV_{r_0}}$ is connected and $r \geq r_0$. We prove $\calG_{\calV_{r}}$ is connected by induction on $|\calV_r|-|\calV_{r_0}|$. If $|\calV_r|=|\calV_{r_0}|$, then there is nothing to show. Suppose there is an $r_0\leq r'< r$ such that $|\calV_{r'}|-|\calV_{r_0}| = |\calV_r|-|\calV_{r_0}|-1$. By the induction hypothesis, $\calG_{\calV_{r'}}$ is connected. Let $\calV_r\backslash \calV_{r'}$ be the singleton set containing $v'$. 

    As we assume that $\widetilde{\boldf}_u\neq \widetilde{\boldf}_v$ for $u,v \notin \calO_0\cup \calO_1$, the level-component $\calC_v = \{v\}$ for any $v\notin \calO_0\cup \calO_1$. By \cref{prop:fev}, there is a path $\calP$ with increasing $\widetilde{\boldf}$ value connecting $\calO_0$ and $\calO_1$ that passes through $v'$. The node before $\calP$ must belong to $\calV_{r'}$, and hence $v'$ is connected to $\calG_{\calV_{r'}}$. As a result, $\calG_{\calV_r}$ is connected. 

    The proof regarding connectedness of $\calG_{\calV_{r}^c}$ is identical.
\end{proof}

We proceed to prove \cref{cor:lvb} with the following observation. 

\begin{Lemma} \label{lem:lsb}
Let $\calO'$ be the immediate neighbors of $\calO$ in $V\backslash \calO$. For each $v'\in \calO'$, let $d_{v'}$ be the number of nodes $v\in \calO$ such that $(v,v')\in \calE$. Then we have
\begin{align*}
    \sum_{v'\in \calO'}d_{v'} \widetilde{\boldf}_{v'} = \sum_{v'\in \calO'}\sum_{(v,v')\in \calE,v\in \calO} \boldf_v.
\end{align*}
\end{Lemma}

\begin{proof}
    For each $v\in \calV\backslash \calO$, by \cref{lem:fev}, $\sum_{(v,v')\in \calE} \widetilde{\boldf}_v = \sum_{(v,v')\in \calE} \widetilde{\boldf}_{v'}$. If we sum up these identities, we have 
    \begin{align}\label{eq:vvf}
    \sum_{v\notin \calO}\sum_{(v,v')\in \calE} \widetilde{\boldf}_v = \sum_{v\notin \calO}\sum_{(v,v')\in \calE} \widetilde{\boldf}_{v'}.
    \end{align} 
    In (\ref{eq:vvf}), if both $v$ and $v'$ are not in $\calO$ and $(v,v')\in \calE$, then $\widetilde{\boldf}_v$ and $\widetilde{\boldf}_{v'}$ occur in both sides of (\ref{eq:vvf}). Canceling these common terms, we have $\sum_{v'\in \calO'}d_{v'} \widetilde{\boldf}_{v'} = \sum_{v'\in \calO'}\sum_{(v,v')\in \calE,v\in \calO} \boldf_v$.
\end{proof}

The proof does not require that $\calO$ in \cref{lem:lsb} be exactly $\calO_0\cup \calO_1$. The only requirement is that we interpolate $\widetilde{\boldf}$ by requiring that it must agree with observed $\boldf$ at $\calO$, so that the averaging property can be applied. 

\begin{proof}[Proof of \cref{cor:lvb}]
    Let $\calG_0$ (resp. $\calV_1$) be the induced subgraph of $\calG$ with nodes $\partial \calV_1 \cup (V\backslash \calV_1)$ (resp. $\partial \calV_0 \cup (V\backslash \calV_0)$). We apply \cref{lem:lsb} to $\partial \calV_1 \cup \calO_0$ in $\calG_0$ and $\partial \calV_0 \cup \calO_1$ in $\calG_1$ respectively to obtain the following identities:
    \begin{align} \label{eq:o0}
        A(\widetilde{\boldf},\calV_0) + \sum_{(v,v')\in \Gamma(\calO_0),v\in \calO_0}\widetilde{\boldf}_{v'}= A(\widetilde{\boldf},\calV_1), \text{ and }
    \end{align}
   \begin{align} \label{eq:o1}
        A(\widetilde{\boldf},\calV_0) + \sum_{(v,v')\in \Gamma(\calO_1),v\in \calO_1}1= A(\widetilde{\boldf},\calV_1) + \sum_{(v,v')\in\Gamma(\calO_1),v\in \calO_1}\widetilde{\boldf}_{v'}.
    \end{align}
As $0<\widetilde{\boldf}_v<1$ for $v\notin \calO$, (\ref{eq:o0}) yields that $A(\widetilde{\boldf},\calV_0) \leq A(\widetilde{\boldf},\calV_1) - a\Gamma(\calO_0)$, while (\ref{eq:o1}) yields that $A(\widetilde{\boldf},\calV_0) \leq A(\widetilde{\boldf},\calV_1) - b\Gamma(\calO_1)$. By combining the two inequalities, we obtain the desired result. 
\end{proof}

It remains to discuss \cref{thm:fcc}. We first prove the result.

\begin{proof}[Proof of \cref{thm:fcc}]
Let $\calU \subset \calV$ be a subset that realizes $h(\cal{G})$. If $\calU\cap \calV' = \emptyset$ or $\calU^c\cap \calV' = \emptyset$, we have the following possibilities.  

\emph{Case 1}: $\calU = \calU_0$ or $\calU = \calU_1$. For the former, we have \begin{align*}h(\calG) = \frac{|\Gamma(\calU_0)|}{\min(|\calU_0|, |\calU_1|+|\calV'|)} \geq \frac{|\Gamma(\calU_0)|}{|\calU_0|}.\end{align*}
Similarly, if $\calU = \calU_1$, then $h(\calG) \geq |\Gamma(\calU_1)|/|\calU_1|$.

\emph{Case 2}: $\calU$ is a proper subset of $\calU_0$. Let $\overline{\calU}$ be the complement of $\calU$ in $\calU_0$. The set $\Gamma(\calU)$ consists of two parts $\calC_1\neq \emptyset$ and $\calC_2 \subset \Gamma(\calU)$. The set $\calC_1$ is a cut in the induced graph $\calG_{\calU_0}$. Therefore, we may estimate $h(\calG)$ as follows:
\begin{align*}
    h(\calG) = \frac{|\calC_1|+|\calC_2|}{\min(|\calU|, |\overline{\calU}|+|\calU_1|+|\calV'|)} \geq \frac{|\calC_1|}{|\calU|} \geq \frac{c(\calG_{\calU_0})}{|\calU_0|}.
\end{align*}

\emph{Case 3}: $\calU$ is a proper subset of $\calU_0$. Similar to Case 2, we have \begin{align*}h(\calG) \geq \frac{c(\calG_{\calU_1})}{|\calU_1|}.
\end{align*}

On the other hand, let $\calN_0$ (resp.\ $\calN_1$) be union $\calU_0$ (resp.\ $\calU_1$) and its the neighbors contained in $\calV'$. By our assumptions, $\calN_0\cap \calN_1 = \emptyset$. For $\calU = \calN_0$, notice that $\Gamma(\calU)$ is a cut of $\calG_{\calV'}$ and we estimate:
\begin{align*}
    h(\calG) \leq \frac{|\Gamma(\calU)|}{\min(|\calN_0|,|\calN_1|)} \leq \frac{C(\calG_{\calV'})}{\min(|\calN_0|,|\calN_1|)}. 
\end{align*}
Therefore, if $C(\calG_{\calV'})$ satifies 
\begin{align*}
    C(\calG_{\calV'}) < c_0 = \min(|\calN_0|,|\calN_1|)\cdot&\min(\frac{|\Gamma(\calU_0)|}{|\calU_0|}, \frac{|\Gamma(\calU_1)|}{|\calU_1|}, \frac{c(\calG_{\calU_0})}{|\calU_0|}, \frac{c(\calG_{\calU_1})}{|\calU_1|}) \\ &(\text{ remark: $c_0$ is independent of $\calG_{\calV'}$ }),
\end{align*}
then Case 1 - Case 3 are impossible and we must have both $\calU\cap \calV' \neq \emptyset$ and $\calU^c\cap \calV' \neq \emptyset$. In this case, we have seen that $h(\calG) \leq  C(\calG_{\calV'})/c_1$, with $c_1=\min(|\calN_0|,|\calN_1|)$.
\end{proof}

We illustrate \cref{thm:fcc} with \figref{fig:bn}. To create a bottleneck for $\calG$ depicted in (a), we may create a bottleneck in $\calV'$ according to the theorem as in (b). The conditions of the theorem may not always be satisfied. For example in (c), ${\calV'}^c$ can be connected. In this case, if $\calV'$ is large enough, then including a small number of additional nodes makes the conditions hold. It is also possible that ${\calV'}^c$ has more than two components as in (d), we can then apply \cref{thm:fcc} repeatedly each time dealing with two components. 

\begin{figure}
    \centering
\includegraphics[scale=0.8]{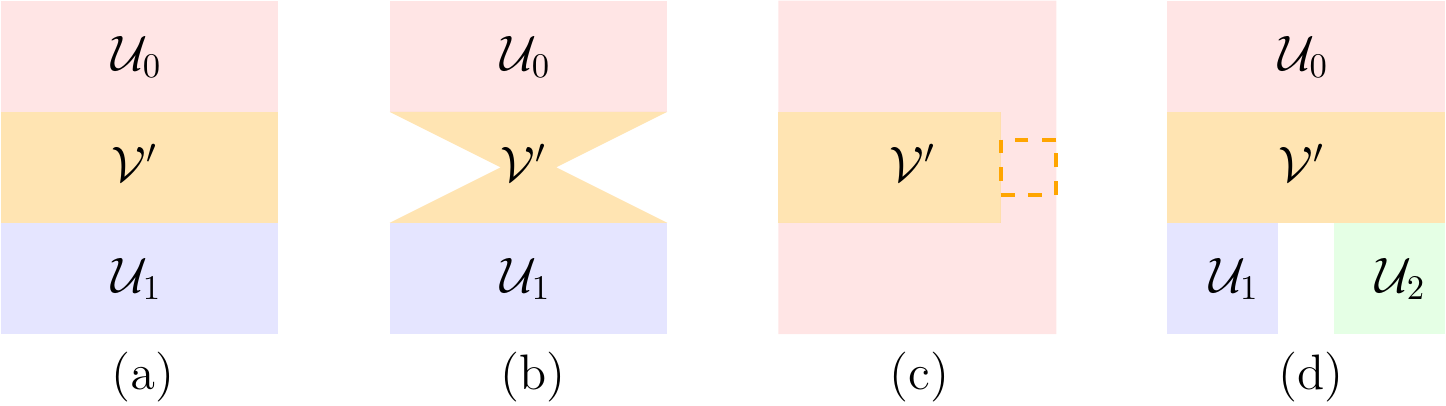}
    \caption{Illustration of \cref{thm:fcc} with Venn diagrams.}
    \label{fig:bn}
\end{figure}

\section{Dataset information and source code} \label{app:sta}
In \cref{tab:ds}, we provide statistics of datasets used in the paper. For data splitting, we follow the cited references. More specifically, Cora, Citeseer, Pubmed, CS, and Photo use 20 training examples per class, with 500 validation samples and 1000 test samples. Disease and Airport use random 30/10/60, 70/15/15 splits respectively. Both Texas and Chameleon use 48/32/20 split. 

\begin{table}[!ht]
\caption{Dataset statistics}
\label{tab:ds}
\begin{center}
\scalebox{1}{\begin{tabular}{  c c c c c } 
  \toprule
  Dataset & Nodes & Edges & Classes & Features \\
 \midrule
  Cora & 2708 & 5429 & 7 & 1433 \\
 \midrule 
  Citeseer & 3327 & 4732 & 6 & 3703 \\
  \midrule 
  Pubmed & 19717 & 44338 & 3 & 500 \\
  \midrule
  Photo & 7487 & 119043 & 8 & 745 \\
  \midrule
  CS & 18333 & 81894 & 15 & 6805 \\
  \midrule
  Disease & 1044 & 1043 & 2 & 1000 \\
  \midrule
  Airport & 3188 & 18631 & 4 & 4 \\
  \midrule
  Texas & 183 & 309 & 5 & 1703 \\
  \midrule
  Chameleon & 2277 & 36101 & 4 & 2325 \\
 \bottomrule
\end{tabular}} 
\end{center}
\end{table}

In \url{http://github.com/amblee0306/label-non-uniformity-gnn}, we provide the source code and instructions to use the code.

\section{More comparisons} \label{sec:mc}
In this appendix, we make direct comparisons with a few more benchmarks: Renode \citep{Che21}, Self-train \citep{Liq18}, and PTDNet \citep{Luo21}. As the implementations of Renode and PTDNet provided by the respective authors are both based on GCN, we also use the {\bf GCN version} of $w$GNN. Therefore unlike \cref{sec:cla}, $w$GNN and its variants in this appendix are not based on the models they compare with. 

Self-training shares some common features with \cref{algo:pab} by introducing test nodes to the training set, we compare it with $w$GNN1 (cf.\ \cref{sec:as}) that does not use \cref{algo:edu}. Similarly, PTDNet involves edge dropping (for a different purpose), and we compare it with $w$GNN2 without using \cref{algo:edu}. As \cref{algo:edu} does not play a role if the graph is a tree or very sparse, the comparisons between $w$GNN2 and PTDNet do not consider Diseases and Texas datasets.

Comparison results are shown in \cref{tab:mc}. We see that $w$GNN (resp.\ $w$GNN1) has an overall better performance than Renode (resp.\ Self-training). On the other hand, each of $w$GNN2 and PTDNet has its own advantages over certain datasets. For the densest datasets Airport, Chameleon, and Photo, $w$GNN2 has much better performance (each with $>10\%$ improvement). For sparser graphs, \cref{algo:edu} does not permit dropping too many edges as we still need to maintain a spanning tree among nodes selected by $w(\cdot)$, and hence it is less impactful. 

\begin{table}[h]
\caption{Comparisons with Renode, Self-train and PTDNet} \label{tab:mc}
\centering
\scalebox{0.75}{
\begin{tabular}{ c c c c c c c c c c c } 
 \toprule
  &  Cora &  Citeseer &  Pubmed &  CS &  Photo &  Airport &  Disease &  Texas &  Chameleon \\ 
  \midrule
  $w$GNN & $83.12 \pm 0.31$ & $73.95 \pm 0.46$ & $80.48 \pm 0.25$ & $89.29 \pm 0.14$ & $92.35 \pm 0.18$ & $87.77 \pm 1.57$ & $89.02 \pm 4.33$ & $58.11 \pm 9.72$ & $64.45 \pm 2.40$ \\ 
 \hdashline
  Renode & $81.28 \pm 0.75$ & $69.54 \pm 0.79$ & $79.62 \pm 0.38$ & $90.08 \pm 0.56$ & $89.25 \pm 1.20$ & $76.40 \pm 3.73$ & $83.35 \pm 1.26$ & $58.11 \pm 4.72$ & $52.98 \pm 2.84$ \\ 
    \midrule
  $w$GNN1 & $82.88 \pm 0.41$ & $73.55 \pm 0.46$ & $80.36 \pm 0.39$ & $89.29 \pm 0.14$ & $91.93 \pm 0.25$ & $87.55 \pm 1.89$ & $89.02 \pm 4.33$ & $55.41 \pm 7.35$ & $63.95 \pm 1.94$ \\ 
 \hdashline
Self-training & $82.27 \pm 0.33$ & $73.24 \pm 0.44$ & $80.32 \pm 0.18$ & $88.92 \pm 0.19$ & $90.62 \pm 0.43$ & $85.69 \pm 0.89$ & $87.18 \pm 1.18$ & $54.14 \pm 7.45$ & $63.97 \pm 2.33$ \\ 
\midrule
  $w$GNN2 & $81.10 \pm 0.25$ & $71.98 \pm 0.52$ & $79.19 \pm 0.37$ & $88.54 \pm 0.37$ & $91.23 \pm 0.30$ & $86.07 \pm 1.52$ & $-$ & $-$ & $64.28 \pm 1.94$ \\ 
 \hdashline
  PTDNet & $82.80 \pm 2.60$ & $72.70 \pm 1.80$ & $79.80 \pm 2.40$ & $90.37 \pm 0.17$ & $80.11 \pm 0.44$ & $64.28 \pm 1.94$ & $-$ & $-$ & $50.35 \pm 1.93$ \\ 
 \bottomrule
\end{tabular}}
\end{table}

%%%%%%%%%%%%%%%%%%%%%%%%%%%%%%%%%%%%%%%%%%%%%%%%%%%%%%%%%%%%%%%%%%%%%%%%%%%%%%%
%%%%%%%%%%%%%%%%%%%%%%%%%%%%%%%%%%%%%%%%%%%%%%%%%%%%%%%%%%%%%%%%%%%%%%%%%%%%%%%

\end{document}

% This document was modified from the file originally made available by
% Pat Langley and Andrea Danyluk for ICML-2K. This version was created
% by Iain Murray in 2018, and modified by Alexandre Bouchard in
% 2019 and 2021 and by Csaba Szepesvari, Gang Niu and Sivan Sabato in 2022.
% Modified again in 2023 by Sivan Sabato and Jonathan Scarlett.
% Previous contributors include Dan Roy, Lise Getoor and Tobias
% Scheffer, which was slightly modified from the 2010 version by
% Thorsten Joachims & Johannes Fuernkranz, slightly modified from the
% 2009 version by Kiri Wagstaff and Sam Roweis's 2008 version, which is
% slightly modified from Prasad Tadepalli's 2007 version which is a
% lightly changed version of the previous year's version by Andrew
% Moore, which was in turn edited from those of Kristian Kersting and
% Codrina Lauth. Alex Smola contributed to the algorithmic style files.